\documentclass[12pt]{article}
\usepackage[margin=1in]{geometry} 
\usepackage{amsmath,amsthm,amssymb}
\usepackage{enumerate}
\usepackage{dsfont}
\usepackage{diagbox}
\usepackage{graphicx}
\usepackage{hyperref}
\usepackage{abstract}
\usepackage{subcaption}
\usepackage{color,soul}

\usepackage[many]{tcolorbox}
\newtcolorbox{revise}{colback=red!5!white, colframe=red!75!black}

\allowdisplaybreaks

\theoremstyle{theorem}
\newtheorem{theorem}{Theorem}
\newtheorem{corollary}[theorem]{Corollary}
\newtheorem{lemma}[theorem]{Lemma}
\newtheorem{proposition}[theorem]{Proposition}
\newtheorem{conjecture}[theorem]{Conjecture}

\theoremstyle{definition}
\newtheorem{definition}[theorem]{Definition}
\newtheorem{example}[theorem]{Example}

\begin{document}
\title{Using activation histograms to bound the number of affine regions in ReLU feed-forward neural networks}
\author{Peter Hinz
}

\maketitle
\begin{abstract}
Several current bounds on the maximal number of affine regions of a ReLU feed-forward neural network are special cases of the framework~\cite{hinz} which relies on layer-wise activation histogram  bounds. We analyze and partially solve a problem in algebraic topology the solution of which would fully exploit this framework. Our partial solution already induces slightly tighter bounds and suggests insight in how parameter initialization methods can affect the number of regions. Furthermore, we extend the framework to allow the composition of subnetwork instead of layer-wise activation histogram bounds to reduce the number of required compositions which negatively affect the tightness of the resulting bound.
\end{abstract}
\section{Introduction}
While today's machine learning success is largely driven by deep neural networks, the reason for their superior performance is still not sufficiently understood. Feed-forward neural networks using the rectified linear unit (ReLU) activation function $x\mapsto x^+=\max(0,x)$ are often used in practice and offer an interesting object to study since they are piece-wise affine-linear on convex regions referred to as \emph{affine regions}. Recent research covers both, empirical and theoretical studies on these regions. In \cite{zhang2020empirical} the local properties such as directions of the corresponding hyperplanes and the decision boundaries for networks trained with different optimization algorithms were analyzed. Of particular interest is also number of affine regions due to the following reasons:
\begin{itemize}
	\item The number of regions are a natural measure of function complexity and expressivity. Insight in the dependency of the number of affine regions on network architectures can be helpful for the construction of networks better suited for specific purposes. For example the work~\cite{pmlr-v119-xiong20a} analyses the number of regions per parameter in convolutional neural networks (CNNs) and compares the results to fully connected networks to conclude a higher expressive power per parameter.
	\item A detailed analysis on the number or regions on initialization, after training and their theoretical maximum might provide an approach to explain the success of deep learning in practice. Works such as \cite{hanin2019complexity} and \cite{hanin2019deep} suggest that practically, the theoretical maximum is not achieved even after training concluding that deep neural networks typically do not exploit their full expressivity thus avoiding overfitting to some extent.
	\item Since the activation function and the network weight initialization play a crucial role for the training process and the quality of the trained network~\cite{HowToStartTraining}, \cite{ImpactActivation}, proposals for new parameter initialization techniques and their theoretical properties are topics of active research~\cite{DNNDecisionTrees}, \cite{kumar2017weight}. Concerning deep ReLU networks, it has been noted that starting with large flexibility (over-parameterization) at initialization is beneficial~\cite{Zou2020} for training process whereas parameter initializations inducing vanishing information (dying ReLU) are harmful~\cite{CiCP-28-1671}. 
		In this light, it is a natural question to ask how the number of regions of the initialized network can be maximized. The search for tighter upper bounds on the number of regions in our work combines sophisticated methodology from several fields such as linear algebra, hyperplane arrangements and algebraic topology, and provides deeper insight in how layers have to be combined and parameters have to be chosen to achieve a high number of regions. Such insight might be relevant for the construction of new parameter initialization methods. For example, our analysis suggest that layer parameters (weights and biases) inducing a ``hot center'' where many neurons are active at the same time, favor more regions.
\end{itemize}

The use of activation histograms has proven to provide a fruitful ansatz to upper bound the number of affine regions of fully connected ReLU feed-forward neural networks. In the work~\cite{hinz} we presented a framework that generalizes previous bounds~\cite{Montufar:2014:NLR:2969033.2969153}, \cite{DBLP:BoundingCounting} and allows to construct even tighter bounds. Its basic idea is to push an upper bound histogram on the histogram of the regions' dimensionalities through the layers starting with a histogram representing the input space. When passing through a network layer, this upper bound histogram is affected by linear transformations which are based on a worst-case analysis of how many regions of which dimensionality can arise from a region of a certain dimension when passed through this network layer. To this end, it is necessary to analyze activation histograms of the hyperplanes induced by the kink of the ReLU activation function used in each neuron.

In this work we want to further exploit and extend the aforementioned framework. This requires tighter worst-case bounds on activation histograms and leads to a problem in the fields of algebraic topology and hyperplane arrangements: What is the tightest upper bound on the activation histograms induced by $p_1$ oriented hyperplanes in $\mathbb{R}^{n_0}$ for $n_0,n_1\in \mathbb{N}_+$? The solution of this problem in the form of explicit histograms will induce the best bound obtainable via framework~\cite{hinz}, thus making it obsolete as a framework. We will examine and solve this problem for $n_1=1$ and also for $n_0\ge n_1$. Via a recursive bound on the activation histograms it is possible to conclude slightly tighter bounds on the number of regions than previously done using the framework. For the case $n_1=2$ we motivate a conjecture for the tightest histogram bounds. Both, our solution for the case $n_1=1$ and our conjecture suggest that a high number of regions is achieved when the transition functions from one layer to the next form a ``hot center'', a central region where all neurons are active.

The bounds on the number of regions obtained from the framework are not tight, even when optimal worst-case activation histograms bounds are used. For every layer transformation of the histogram of region dimensionalities, the worst-case activation histogram is used to bound the histogram of region dimensionalities in the next layer. However, it is too conservative to assume the worst case jointly for all regions such that for every composition, tightness is lost. To reduce the number of compositions necessary, we extend the previous framework to allow the composition of worst-case activation histogram bounds corresponding to whole subnetworks instead of such bounds per layer.

This work is structured as follows. We briefly introduce necessary mathematical definitions and concepts in Section~\ref{sec:preliminaries}. In Section~\ref{sec:firstOrder} we focus on the algebraic topology problem that allows to fully exploit the previous framework~\cite{hinz} and present our partial solution together with a conjecture for input dimension $p_0=2$. A recursive histogram bound allows us to make use of the cases $p_0=1$ and $p_0=2$ for larger input dimensions $p_0$ such that we are able to provide slightly improve on previous bounds on the number of regions. Then we generalize the framework for the composition of subnetwork activation histogram bounds in Section~\ref{sec:higherOrder}. In Section~\ref{sec:summary} we summarize our results.

\section{Preliminaries}
\label{sec:preliminaries}
We denote the non-negative and positive integers by $\mathbb{N}$, $\mathbb{N}_+$ respectively. For a depth $L\in\mathbb{N}_+$, a feed-forward neural net $f:\mathbb{R}^{n_0}\to\mathbb{R}^{n_L}$ with ReLU activation functions, input dimensionality $n_0\in\mathbb{R}$ hidden layer widths $n_1,\dots,n_{L-1}\in\mathbb{N}_+$ and output layer width $n_L\in\mathbb{N}_+$ is a function composition of the form
\begin{equation}
	\label{eq:neuralNet}
	f=f_{\mathbf{h}}=h_L\circ\dots\circ h_1
\end{equation} of \emph{layer transition functions} $\mathbf{h}=(h_1,\dots,h_L)$ with $h_i\in\textnormal{RL}(n_{i-1},n_i)$ for $i\in\left\{ 1,\dots,L \right\}$,
\begin{equation}
	\textnormal{RL}(p_0,p_1)=\left\{h:
		\begin{aligned}
			\mathbb{R}^{p_0}&\to\mathbb{R}^{p_1}\\
			x&\mapsto \textnormal{ReLU}(W_hx+b_h)
		\end{aligned}
	\middle\vert 
	\begin{array}{l}
	W_h\in\mathbb{R}^{p_i\times p_{i-1}}\\
b_h\in\mathbb{R}^{p_i}
	\end{array}\right\}\quad p_0,p_1\in\mathbb{N}_+
\end{equation}and activation function $\textnormal{ReLU}:\mathbb{R}\to\mathbb{R}, x\mapsto \max(0,x)$ which is applied component-wise. Here, $W_h$ is called the \emph{weight matrix} and $b_h$ the \emph{bias vector}. Usually a final affine map is applied to the output of the last ReLU layer, however such a map cannot increase the number of regions such that we omit it in our definition in equation~\eqref{eq:neuralNet}. 

For $p_0,p_1\in\mathbb{N}_+$, a layer transition function $h\in\textnormal{RL}(p_0,p_1)$ and $i\in\left\{ 1,\dots,p_1 \right\}$ we say that the $i$-th neuron is \emph{active} at input $x\in\mathbb{R}^{p_0}$ if $h(x)_i>0$ otherwise, it is \emph{inactive}. We can encode this activity of the neurons in a binary $p_1$-tuple $S_h(x)\in\left\{ 0,1 \right\}^{p_1}$ via the condition $S_h(x)_i=1$ if and only if $h(x)_i>0$. We call $S_h(x)\in\left\{ 0,1 \right\}^{p_1}$ the \emph{activation pattern} of $h$ at $x$ and $\mathcal{S}_h=\left\{ S_h(x)\mid\;x\in\mathbb{R}^{p_0} \right\}$ the \emph{attained activation patterns} of $h$. The function $S_h(x)$ partitions its input space into different regions of constant activation pattern. These regions are separated by the $p_1$ possibly degenerated hyperplanes $H_i=\left\{ x\in\mathbb{R}^{n}\mid\; (W_hx+b_h)_i=0 \right\}$ for $i\in\left\{ 1,\dots,p_1 \right\}$. In other words, the number of attained activation patterns $|\mathcal{S}_h|$ is equal to the number of regions that the space is partitioned into by these hyperplanes. The sharp bound 
\begin{equation}
	\label{eq:schlaefliBound}
	|\mathcal{S}_h|\le\sum_{j=0}^{\min(p_0,p_1)}\tbinom{p_1}{j}
\end{equation}
for this number was discovered by L. Schläfli \cite{Schlafli1950} and is attained when the hyperplanes are non-degenerated and are in an arrangement called \emph{general position}. Since $h$ is itself of the form~\eqref{eq:neuralNet}, this gives a sharp upper bound for the number of regions for ReLU neural networks for $L=1$. For $L>1$ we can define the activation pattern by
\begin{equation}
	S_{\mathbf{h}}(x)=(S_{h_1}(x),S_{h_2}(h_1(x)),\dots,S_{h_L}(h_{L-1}\circ\dots\circ h_1(x)))\in\left\{ 0,1 \right\}^{n_1}\times\dots\times \left\{ 0,1 \right\}^{n_L},
\end{equation}
which is the $L$-tuple of the activation patterns of the individual layer transition functions. It specifies which neurons of the neural net $f_\mathbf{h}$ are active or inactive for some input $x\in\mathbb{R}^{n_0}$. Similarly we define the attained activation patterns of $f_\mathbf{h}$ by $\mathcal{S}_{\mathbf{h}}=\left\{ S_{\mathbf{h}}(x)\mid\;x\in\mathbb{R}^{n_0} \right\}$. Note that for $i\in\left\{ 1,\dots,L \right\}$, the transition function $h_i$ satisfies
\begin{equation*}
	h_i(x)=\textnormal{ReLU}(W_{h_i}x+b_{h_i})=\textnormal{diag}(S_{h_i}(x))\left( W_h x+b_h \right)\quad \textnormal{ for }x\in\mathbb{R}^{n_{i-1}}.
\end{equation*}In particular for every $s\in\left\{ 0,1 \right\}^{n_1}\times\dots\times \left\{ 0,1 \right\}^{n_L}$ and all $x\in\mathbb{R}^{n_0}$ such that $S_{\mathbf{h}}(x)=s$, the neural net $f_{\mathbf{h}}$ satisfies 
$f_{\mathbf{h}}(x)=\tilde{h}^{(s_L)}_L\circ\dots\circ\tilde{h}_{1}^{(s_1)}(x)$ for affine linear functions $\tilde{h}^{(s_i)}_i:\mathbb{R}^{n_{i-1}}\to\mathbb{R}^{n_i}, x\mapsto \textnormal{diag}(s_i)(W_{h_i}x+b_{i})$. This means that $f_{\mathbf{h}}$ is affine linear on every set of the form $\left\{ x\in\mathbb{R}^{n_0}\mid\;S_{\mathbf{h}}(x)=s \right\}$ for $s\in\mathcal{S}_{\mathbf{h}}$ and hence we define the \emph{number of affine-linear regions} of a neural network $f_{\mathbf{h}}$ by $|\mathcal{S}_{\mathbf{h}}|$. In this paper, we are aiming to find bounds $N(n_0,\mathbf{n})$ that satisfy
\begin{equation}
	\label{eq:goal}
	|\mathcal{S}_{\mathbf{h}}|\le N(n_0,\mathbf{n})\quad\textnormal{ for all }\mathbf{h}\in\textnormal{RL}(n_0,\mathbf{n})
\end{equation}
with $\textnormal{RL}(n_0,\mathbf{n})=\textnormal{RL}(n_0,n_1)\times\dots\times\textnormal{RL}(n_{L-1},n_L)$ and $n_{0},L\in\mathbb{N}_+$, $\mathbf{n}=(n_1,\dots,n_L)\in\mathbb{N}_+^L$. In the next Section~\ref{sec:firstOrder} we will use the framework~\cite{hinz} to improve on existing bounds using a layer-wise histogram based approach to count the number of regions and their dimensionality. As a byproduct our approach also provides insight in how the arrangement of the layer-generated hyperplanes control the number of regions. In Section~\ref{sec:higherOrder} we will present a generalization of this framework that allows the composition on subnetwork instead of layer-wise histograms bounds. 
\section{Composing layer-wise activation histograms}
\label{sec:firstOrder}
\subsection{Previous framework and results}
\subsubsection{Intuitive background and definitions}
We first briefly recall the intuition behind and the main results of framework~\cite{hinz} to which the reader may be referred for details.  It relies on layer-wise bounds on the histogram of the number of active neurons of a layer transition function. More precisely, let $p_0,p_1\in\mathbb{N}_+$ be the input and output dimension of a layer transition function $h\in\textnormal{RL}(p_0,p_1)$. Note that $h$ is affine linear $h(x)=\tilde{h}^{(s)}(x)$ on $\left\{ x\in\mathbb{R}^{n}\mid\; S_h(x)=s\right\}$ for each $s\in\mathcal{S}_h$ with $\tilde{h}^{(s)}:\mathbb{R}^{p_0}\to\mathbb{R}^{p_1}, x\mapsto\textnormal{diag}(s)(W_hx+b_h)$.  Using $|s|=\sum_{i=1}^{p_1}s_i$, the rank of the affine linear map $\tilde{h}^{(s)}$ is bounded by $\min(p_0,|s|)$, i.e.
\begin{equation}
	\label{eq:rankBound}
	\textnormal{Rank}(\tilde{h}^{(s)})\le \min(p_0,|s|)
\end{equation}
Now let $n_0,L\in\mathbb{N}_+$, $\mathbf{n}=(n_1,\dots,n_L)\in\mathbb{N}_+^L$ and $\mathbf{h}\in\textnormal{RL}(n_0,\mathbf{n})$. The fundamental idea for the construction of upper bounds on $|\mathcal{S}_{\mathbf{h}}|$ is the above rank bound and an analysis of the transformation of a dimension histogram when it is pushed through the layers of the neural network. 

Before we go into detail, we need to define the set of such histograms as $V:=\big\{v\in\mathbb{N}^\mathbb{N}\mid\;\sum_{j=0}^\infty v_j<\infty\big\}$ with the canonical elements ${\rm e}_i$, $i\in\mathbb{N}$ defined by $({\rm e}_i)_j=\delta_{ij}$ for $i,j\in \mathbb{N}$. With this definition each $v\in V$ satisfies $v=\sum_{i=0}^\infty v_i {\rm e}_i$. Furthermore we can introduce an order relation ``$\preceq$'' on $V$ by
\begin{equation*}
	v\preceq w\iff \forall J\in\mathbb{N}:\;\sum_{j=J}^\infty v_j\le \sum_{j=J}^{\infty}w_j\quad \textnormal{ for }v,w\in V.
\end{equation*}With this definition $V$ becomes a \emph{join-semilattice}: For every two histograms $v,w\in V$ there exists a smallest upper bound histogram $v\vee w\in V$, called \emph{join} which is determined by the conditions $\sum_{j=J}(v\vee w)_j=\max(\sum_{j=J}^\infty v_j,\sum_{j=J}^\infty w_j)$ for $J\in\mathbb{N}$. The join $\bigvee_{i\in I} v_i$ of finitely many elements $(v_i)_{i\in I}$ is defined inductively. In our interpretation, an element $v\in V$ counts the number of regions $v_i$ of each dimensionality $i\in\mathbb{N}$. 

In this sense, we can encode the input space as one region of dimension $n_0$, i.e. as ${\rm e}_{n_0}$. This dimension histogram is then pushed through the layers with certain transformation rules $\varphi_l:V\to V$, $l\in\left\{ 1,\dots,L \right\}$ such that for every layer index $l\in\left\{ 1,\dots,L \right\}$, the histogram $\varphi_l\circ\dots\circ\varphi_1({\rm e}_{n_0})$ bounds the output dimension histogram of the dimensions of the occurring regions of the subnetwork $h_l\circ\dots\circ h_1$ with respect to ``$\preceq$''. 
Since the final bound~\eqref{eq:goal} shall only depend on the network topology $n_0,\dots,n_L$ and not on the actual weights and bias vectors of the neural network, we need to apply a worst case analysis with respect to ``$\preceq$'', i.e. a join in $V$.

To this end we consider a layer transition function $h\in\textnormal{RL}(p_0,p_1)$ for an input dimension $p_0\in\mathbb{N}$ and output dimension $p_1\in\mathbb{N}_+$. Since there may exist regions with dimensionality $0$ we need to allow extend our previous definition for $p_0=0$ by 
\begin{equation*}
	\textnormal{RL}(0,p_1)=\{h:\{0\}\to \mathbb{R}^{p_1}, 0\mapsto \textnormal{ReLU}(b_1)|b_1\in\mathbb{R}^{p_1}\}.
\end{equation*}
By equation~\eqref{eq:rankBound} we can $\preceq$ bound the histogram of occurring output dimensions on the affine regions by
\begin{equation}
	\sum_{s\in\mathcal{S}_h}{\rm e}_{\textnormal{rank}(\tilde h^{(s)})}\preceq \sum_{s\in\mathcal{S}_h}{\rm e}_{\min(p_0,|s|)}=\textnormal{cl}_{p_0}(\sum_{s\in\mathcal{S}_h}{\rm e}_{|s|})
\end{equation}
with the \emph{clipping functions} $\textnormal{cl}_j:V\to V, v\mapsto \sum_{i=0}^j v_i{\rm e}_{\min(i,j)}$ for  $j\in\mathbb{N}$. We call $\sum_{s\in\mathcal{S}_h}{\rm e}_{\textnormal{rank}(\tilde h^{(s)})}$ the \emph{dimension histogram} and $\sum_{s\in\mathcal{S}_{h}}{\rm e}_{|s|}$ the \emph{activation histogram} of $h$. The \emph{activation histogram join} 
	 \begin{equation}
		 \label{eq:activationHistogramJoin}
	 \tau_{p_0}^{p_1}:= \bigvee_{h'\in \textnormal{RL}(p_0,p_1)}\sum_{s\in\mathcal{S}_{h'}}{\rm e}_{|s|}\in V\quad \textnormal{for }p_0,p_1\in\mathbb{N}_+
	 \end{equation}
	with the convention $\tau_{0}^{p_1}:={\rm e_{p+1}}$ for $p_1\in\mathbb{N}_+$ plays an important role in our theory because it allows to construct the tightest worst-case dimension histogram bound on all dimension histograms induced by layer transition functions in $\textnormal{RL}(p_0,p_1)$:
	 \begin{equation}
		 \label{eq:sharpestBound}
		 \forall h \in \textnormal{RL}(p_0,p_1)\;\sum_{s\in\mathcal{S}_h}{\rm e}_{\textnormal{rank}(\tilde h^{(s)})}\preceq\textnormal{cl}_{p_0}(\sum_{s\in\mathcal{S}_h}{\rm e}_{|s|})
	\preceq
	\bigvee_{h'\in \textnormal{RL}(p_0,p_1)}\textnormal{cl}_{p_0}(\sum_{s\in\mathcal{S}_{h'}}{\rm e}_{|s|} )
	=
	\textnormal{cl}_{p_0}\left( \tau_{p_0}^{p_1} \right)
\end{equation}
We will analyze the activation histogram join in detail in Section~\ref{sec:activationHistogramProblem}. By the above inequality for every $\gamma_{p_0,p_1}\in V$ with $\tau_{p_0}^{p_1}\preceq \gamma_{p_0,p_1}$ we can bound the dimension histogram $\sum_{s\in\mathcal{S}_h}{\rm e}_{\textnormal{rank}(\tilde h^{(s)})} \preceq \textnormal{cl}_{p_0}\left(\gamma_{p_0,p_1} \right)$ for $h \in \textnormal{RL}(p_0,p_1)$. 

The idea for the construction of upper bounds on the number of affine regions is based on the layer-wise application of such histogram inequalities for all affine regions induced by the previous layers, where their respective dimensions take the role of the input dimension $p_0$ above for the next layer transformation function. In this sense $\textnormal{cl}_{p_0}\left( \gamma_{p_0,n_{l}} \right)$ is a reasonable choice for the images $\varphi_{l}({\rm e}_{p_0})$ of the transformation rules $\varphi_l$ above for $l\in\left\{ 1,\dots,L \right\}$, $p_0\in\mathbb{N}$ and a collection $\gamma_{p_0,p_1}\in V$ with $\tau_{p_0}^{p_1}\preceq\gamma_{p_0,p_1}$ for $p_0\in\mathbb{N},p_1\in\mathbb{N}_+$. 
	 \subsubsection{Previous framework's result}

\begin{definition}
  \label{def:boundCondition}
  We say that a collection $(\gamma_{p_0,p_1})_{p_1\in\mathbb{N}_+,p_0\in\left\{ 0,\dots,p_1 \right\}}$ of elements in $V$ satisfies the \emph{layer-wise bound condition} if the following statements are true:
  \begin{enumerate}
	  \item $\forall p_1\in\mathbb{N}_+,p_0\in\left\{ 0,\dots,p_1 \right\}\quad \tau_{p_0}^{p_1}\preceq \gamma_{p_0,p_1}$
	  \item $\forall p_1\in \mathbb{N}_+,p_0,\tilde{p}_0\in\left\{ 0,\dots,p_1 \right\}\quad p_0\le \tilde {p}_0\implies \gamma_{p_0,p_1}\preceq \gamma_{\tilde{p}_0,p_1}$
  \end{enumerate}
  The first condition states that the histograms in the collection bound all activation histograms because they bound the activation histogram joins whereas the second condition requires the collection to be increasing in the input dimension which is necessary for our worst-case analysis.
  The set of all such collections $\gamma$ is denoted by 
  \begin{equation}
	\label{eq:Gamma}
	\Gamma:=\left\{ (\gamma_{p_0,p_1})_{p_1\in\mathbb{N}_+,n\in\left\{ 0,\dots,p_1 \right\}}\mid \gamma \text{ satisfies the layer-wise bound condition}  \right\}
  \end{equation}
\end{definition}
We also call these collections \emph{layer-wise activation histogram bounds} because they bound the activation histograms of how often how many neurons are active for the regions induced by layer transition functions.
For $p_1\in\mathbb{N}_+$ and a collection of layer-wise activation histogram bounds $\gamma\in \Gamma$ we define the \emph{transformation rule} 
   \begin{equation}
	   \label{eq:transformationRule}
	   \varphi^{(\gamma)}_{p_1}: V\to V,\;v\mapsto \sum_{p_0=0}^{\infty}v_{p_0} \textnormal{cl}_{\min(p_0,p_1)}(\gamma_{\min(p_0,p_1),p_1}).
   \end{equation}
	 We can now state the main results of the framework~\cite{hinz}.
\begin{theorem}
	\label{thm:mainFirstOrder}
 Every $\gamma\in\Gamma$ induces a bound on the number of regions via
 \begin{equation}
	 \label{eq:firstOrderBound}
   |\mathcal{S}_{\mathbf{h}}|\le\|\varphi^{(\gamma)}_{n_L}\circ\dots\circ\varphi^{(\gamma)}_{n_1}({\rm e}_{n_0})\|_1.
 \end{equation}
 Equivalently, in matrix formulation it holds that
	\begin{equation}
	 \label{eq:MainRecursiveMatrix}
	 |\mathcal{S}_{\mathbf{h}}|\le \|B^{(\gamma)}_{n_L}M_{n_{L-1},n_L}\dots B^{(\gamma)}_{n_1}M_{n_{0},n_{1}}e_{n_0+1}\|_1
	\end{equation}
	with the canonical basis vector $e_{n_0+1}\in\mathbb{R}^{n_0+1}$ and matrices $B^{(\gamma)}_{p_1}\in\mathbb{N}^{(p_1+1)\times (p_1+1)}$ and $M_{p_0,p_1}\in\mathbb{R}^{p_1+1\times p_0+1}$ defined by
   \begin{align}
	   \label{eq:boundMatrix}
	   (B^{(\gamma)}_{p_1})_{i,j}&= \left(\varphi_{p_1}^{(\gamma)}({\rm e}_{j-1})\right)_{i-1}=\left( \textnormal{cl}_{j-1}(\gamma_{j-1,p_1}) \right)_{i-1}\quad i,j\in\left\{ 1,\dots,p_1+1 \right\},\\
	   \label{eq:M}
	   M_{i,j}&= \delta_{i,\min(j,p_1+1)}, \quad i\in\left\{ 1,\dots,p_1+1 \right\},j\in\left\{ 1,\dots,p_0+1 \right\},\\
	   \label{eq:e}
	(e_{n_0+1})_i&= \delta_{n_0+1,i}, \quad i\in\left\{ 1,\dots,n_0+1 \right\}
\end{align}for $p_0,p_1\in\mathbb{N}$.
\end{theorem}
Note that the matrices $B^{(\gamma)}_{p_1}$, $n\in\mathbb{N}$ are upper triangular by definition and hence its eigenvalues can be read from the diagonal. For example for $\gamma\in\Gamma$, $n_0,n,L\in\mathbb{N}_+$ with $n_0\le n$ and $\mathbf{n}=(n,\dots,n)\in\mathbb{N}^L$, $\mathbf{h}\in\textnormal{RL}(n_0,\mathbf{n})$ it holds that 
	\begin{equation}
		\label{eq:exampleBound}
		|\mathcal{S}_h|\le \|\big( B^{(\gamma)}_{n_L} \big)^L M_{n_0,n}e_{n_0+1}\|_1=\mathcal{O}(\max(B^{(\gamma)}_{1,1},\dots,B^{(\gamma)}_{n_0+1,n_0+1})^L)\quad \textnormal{ as }L\to\infty. 
	\end{equation}This bounds the asymptotic order of an equal width feed-forward neural network for $L\to\infty$.
 \subsubsection{Choices for the histogram activation bounds}
 In order to make use of Theorem~\ref{thm:mainFirstOrder} we need to plug in a collection of layer-wise activation histogram bounds $\gamma\in \Gamma$. The work \cite{hinz} presented the collections $\widehat \gamma,\tilde\gamma, \bar\gamma\in\Gamma$ given by 
 \begin{equation}
   \label{eq:oldGammas}
   \widehat{\gamma}_{p_0,p_1}=2^{p_1}{\rm e}_{p_1}, \quad \tilde{\gamma}_{p_0,p_1}=\sum_{j=0}^{p_0}\tbinom{p_1}{ j}{\rm e}_{p_1},\quad \bar{\gamma}_{p_0,p_1}=\sum_{j=0}^{p_0}\tbinom{p_1}{ j}{\rm e}_{p_1-j}
 \end{equation}for $p_1\in\mathbb{N}_{+}, p_0\in\left\{ 0,\dots,p_1 \right\}$. They are ordered from weak to strong in the sense that 
 $ \bar{\gamma}_{p_0,p_1}\preceq \tilde{\gamma}_{p_0,p_1}\preceq\widehat{\gamma}_{p_0,p_1}$ for such $p_0,p_1$. The intuition behind $\widehat{\gamma}$ is that for $h\in\textnormal{RL}(p_0,p_1)$, the $p_1$ induced hyperplanes partition the space $\mathbb{R}^{p_0}$ into at most $2^{p_1}$ regions and on each there are at most $p_1$ active units. The bound~\eqref{eq:schlaefliBound} with the same reasoning yields $\tilde{\gamma}$. The strongest of the above elementary bounds $\bar{\gamma}$ combines equation~\eqref{eq:schlaefliBound} with the fact that there are at most $\tbinom{p_1}{j}=\big\{ s\in\big\{ 0,1 \big\}^{p_1}\mid\;|s|=j \big\}$ regions where $j\in\left\{ 0,\dots,p_1 \right\}$ neurons are active because by definition $\mathcal{S}_h\subset \left\{ 0,1 \right\}^{p_1}$. All three of these bounds yield bounds of previous works~\cite{Montufar:2014:NLR:2969033.2969153} Proposition~3, \cite{Montufar17} Proposition 3 and~\cite{DBLP:BoundingCounting}. Stronger activation histogram bound collections yield stronger composed bounds in Theorem~\ref{thm:mainFirstOrder}. It is therefore of interest to find tighter collections. In Section~\ref{sec:improvedBound} we derive a slightly improved collection $\gamma^*$ that can be used in that theorem. It is based on insight on the activation histogram join which we will present in the next section.

 \subsection{The activation histogram join}
 \label{sec:activationHistogramProblem}
 \subsubsection{Description}
In order to fully exhaust the above framework, it is necessary to find tight activation histogram bounds. The following result states that the histogram join from equation~\eqref{eq:activationHistogramJoin} itself satisfies the bound condition.
 \begin{lemma}
	 \label{lem:tauBoundCondition}
	 The collection $\tau_{p_0}^{p_1}$ for $p_1\in \mathbb{N}$,$p_0\in\left\{ 1,\dots,p_1 \right\}$ satisfies the layer-wise bound condition.
 \end{lemma}
 Despite the fact that $\tau_{p_0}^{p_1}$ is also defined for $p_0\ge p_1$, we will consider $\tau$ as an element of $\Gamma$. With the above result it is clear that the tightest elementary bound collection that can be used in Theorem~\ref{thm:mainFirstOrder} is $\tau$ itself and it will yield new tighter bounds on the number of affine-linear regions of ReLU feed-forward neural networks than previously presented in \cite{hinz}. It is therefore of interest to analyze $\tau$ in detail. We call the problem of finding an explicit formula for the abstract join in equation~\eqref{eq:activationHistogramJoin} the \emph{activation histogram join problem}. 
 
 At this point we want to give some intuition on this problem. 
 For input and output dimension $p_0,p_1\in \mathbb{N}$, by definition $\tau_{p_0}^{p_1}= \bigvee_{h\in \textnormal{RL}(p_0,p_1)}\sum_{s\in\mathcal{S}_{h}}{\rm e}_{|s|}$. This means that $\tau_{p_0}^{p_1}$ is the smallest histogram in $V$ that bounds every $\sum_{s\in\mathcal{S}_{h}}{\rm e}_{|s|}$ for $h\in\textnormal{RL}(p_0,p_1)$. For a specific $h\in\textnormal{RL}(p_0,p_1)$ with weight matrix $W_h\in\mathbb{R}^{p_1\times p_0}$ and the bias vector $b_h\in\mathbb{R}^{p_1}$ there are hyperplanes induced by $H_i=\left\{ x\in\mathbb{R}^{p_0}\mid (W_hx+b_h)_i=0 \right\}$ for row indices $i\in\left\{ 1,\dots,p_1 \right\}$ with non-zero rows. These hyperplanes also have an orientation if we consider that they partition the space into an active and an inactive side by the condition $(W_hx+b_h)_i>0$ or $\le0$ respectively for $x\in\mathbb{R}^{p_0}$. The whole collection of these hyperplanes partitions the space $\mathbb{R}^{p_0}$ at most $\sum_{j=0}^{\min(p_0,p_1)}{p_1\choose j}$ by equation~\eqref{eq:schlaefliBound}. Each of these regions is on the active side of a number of hyperplanes. Now $\sum_{s\in\mathcal{S}_{h}}{\rm e}_{|s|}$ is just the histogram of this number for all occurring regions.

 Therefore, the activation histogram join problem is actually a problem in the field of oriented hyperplane arrangements and, in addition to its relevance to find tighter bounds on the number of regions, is an interesting question on its own in this and related fields of mathematics such as matroid theory and algebraic topology.

 \subsubsection{Solution for input dimension not smaller than output dimension}
If there are less hyperplanes than the dimension of the space $p_1\le p_0$, then all possible activation patterns can be observed at the same time, i.e. $\mathcal{S}_h=\{0,1\}^{n_1}$ such that the histogram counting the number of ones is formed using binomial coefficients.
	\begin{lemma}
		\label{lem:explicitTauPowerset}
		For $p_0,p_1\in\mathbb{N}_+$ with $p_0>p_1$ it holds that $\tau_{p_{0}}^{p_1}=\tau_{p_1}^{p_1}=\sum_{i=0}^{p_1}\tbinom{p_1}{i}{\rm e}_i$.
	\end{lemma}
 \subsubsection{Solution for input dimension 1}
 For input dimension $p_0=1$ and arbitrary output dimension $p_1\in\mathbb{N}_+$ we need to consider histograms generated by $p_1$ oriented points on the real line. 
 \begin{proposition}
	 \label{prop:boundConditionInputOne}
	 For all $p_1\in\mathbb{N}_+$, 
 	 \begin{equation*}
		 \tau_{1}^{p_1}=
		 \begin{cases}
			 {\rm e}_{(p_1-1)/2}+\sum_{i=\lceil p_1/2\rceil}^{p_1-1} 2{\rm e}_{i}+{\rm e}_{p_1}&\quad \textnormal{ if $p_1$ is odd}\\
			 \sum_{i=p_1/2}^{p_1-1}2{\rm e}_{i}+{\rm e}_{p_1}&\quad \textnormal{ if $p_1$ is even}
		 \end{cases}
	 \end{equation*}
	 and there exists $h\in\textnormal{RL}(1,p_1)$ with $\sum_{s\in\mathcal{S}_h}{\rm e}_{|s|}=\tau_{1}^{p_1}$
\end{proposition}
The proof is deferred to Appendix~\ref{sec:inputDimOne}. Intuitively we first show that we only need to consider $p_1$ distinct points, then we prove that we only need to consider configurations with a hot region where all neurons are active and finally we conclude the result by balancing the number of points on the left and right side of the hot region, i.e. placing the hot region in the center, see Figure~\ref{fig:dim1points}. It is remarkable that the join is attained in equation~\eqref{eq:activationHistogramJoin} for some $h\in\textnormal{RL}(1,p_1)$. For $p_0\ge 2$ it is still an open problem and it is not clear if the join is also attained. However, for $p_0=2$ we have a conjecture presented in the next section.
\begin{figure}
  \centering
  \includegraphics[width=0.7\textwidth]{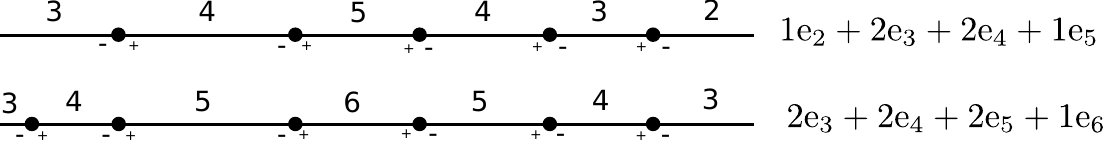}
  \caption{Idea of the proof of Proposition~\ref{prop:boundConditionInputOne}. The maximal histogram is attained for a hot center region in the middle that lies on the positive side of all hyperplanes (points). Above we demonstrate this for the cases $p_1=5$ and $p_1=6$}
\label{fig:dim1points}
\end{figure}
\begin{figure}
  \centering
  \includegraphics[width=\textwidth]{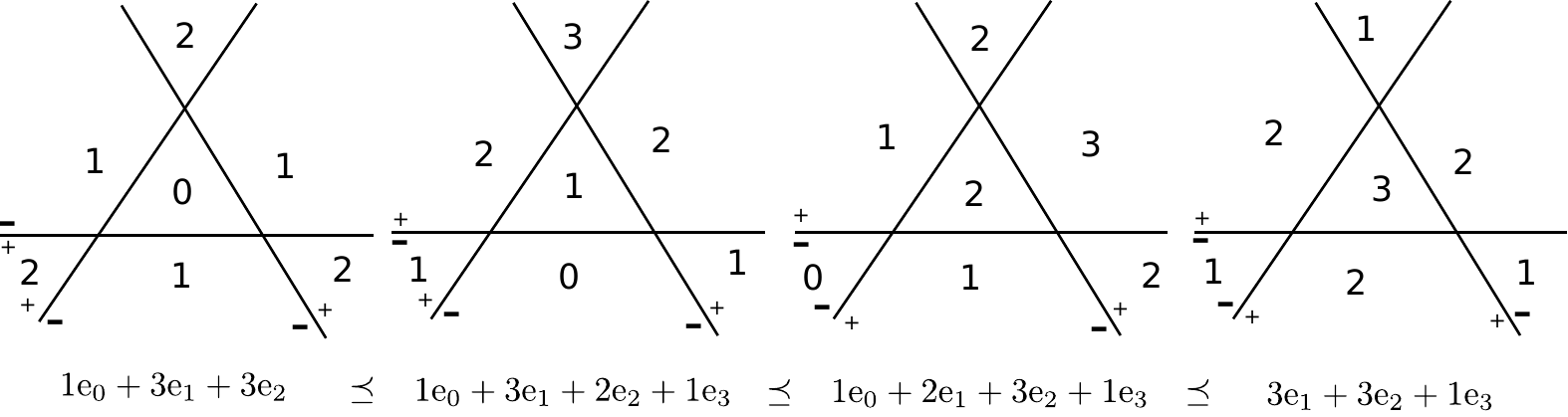}
  \caption{Different configurations and corresponding histograms for 3 lines in $\mathbb{R}^2$ in general position. The plus ($+$) and minus $(-)$ signs indicate the active and inactive side of each hyperplane respectively. The numbers of active hyperplanes is directly written into the regions. This exhaustive search shows that $\tau_{2}^3=3{\rm e}_1+3{\rm e}_3+{\rm e}_3$.}
\label{fig:threelines}
\end{figure}

\begin{figure}
  \centering
  \includegraphics[width=\textwidth]{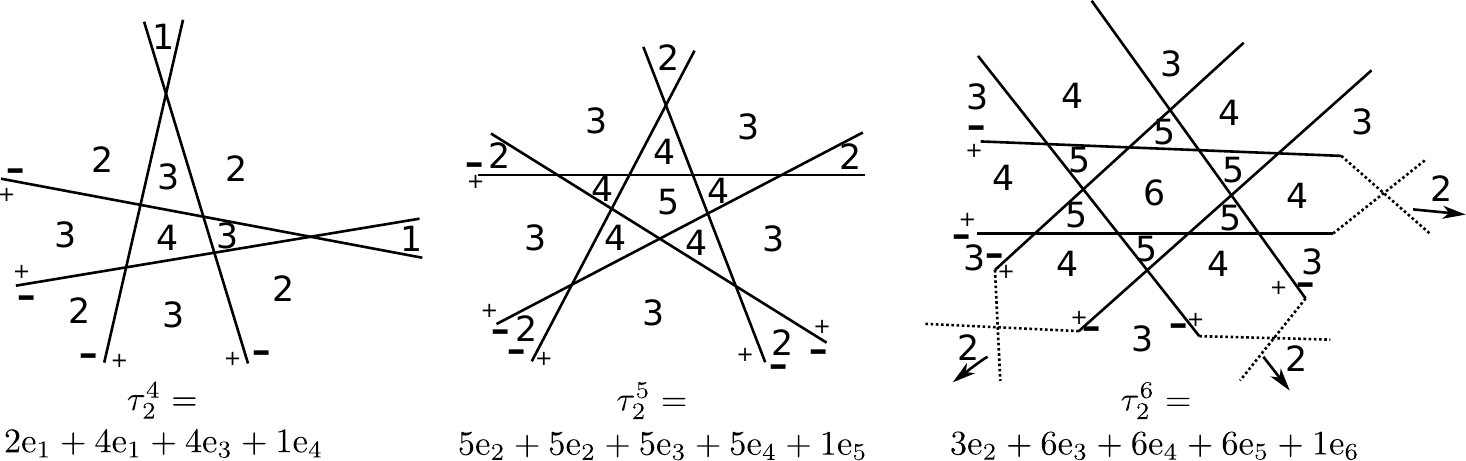}
  \caption{Configurations for 4,5 and 6 lines in general position following the idea of a ``hot center''. Similarly to Figure~\ref{fig:threelines} we found that the corresponding histograms are $\tau_{2}^4,\tau_2^5,\tau_2^6$ by an exhaustive search on all combinations.}
\label{fig:morelines}
\end{figure}
\subsubsection{Conjecture for input dimension 2}
For input dimension $p_0=2$ we need to consider oriented lines in $\mathbb{R}^{2}$. In Figures~\ref{fig:threelines} we show all different configurations for $p_1=3$ oriented lines in $\mathbb{R}^2$ and find that that the maximal activation histogram is attained for a hot region in the center. In Figure~\ref{fig:morelines} we find for $p_1\in\{4,5,6\}$ that again the maximal histogram is attained for a hot center. This leads us to the following conjecture on the activation histogram join.
\begin{conjecture}
	\label{conj:two}
	For all natural $p_1\ge 2$ it holds that
\begin{equation*}
	\tau_{2}^{p_1}=
	\begin{cases}
		\sum_{i=\lfloor p_1/2\rfloor}^{p_1-1}p_1 {\rm e}_i+{\rm e}_{p_1}\quad &\textnormal{if $p_1$ is odd}\\
		\frac{p_1}{2} {\rm e}_{p_1/2-1}+\sum_{i= p_1/2}^{p_1-1}p_1 {\rm e}_i+{\rm e}_{p_1}\quad &\textnormal{if $p_1$ is even}
	\end{cases}
\end{equation*}
and there exist $h\in\textnormal{RL}(2,p_1)$ with $\sum_{s\in\mathcal{S}_h}{\rm e}_{|s|}=\tau_{2}^{p_1}$.
\end{conjecture}
We will now motivate this conjecture. First note that indeed the $\|\cdot\|_1$-norm of the conjecture histograms is correct, i.e. the number of regions encoded in these histograms coincides with the number of regions in formula~\eqref{eq:schlaefliBound} induced by $n_1$ hyperplanes in general position in $\mathbb{R}^{n_0}$:
\begin{equation*}
	\forall p_1\ge 2:\quad
	\sum_{j=0}^{2}\tbinom{p_1}{j}=
	\begin{cases}
		\sum_{i=\lfloor p_1/2\rfloor}^{p_1-1}p_1 +1\quad &\textnormal{if $p_1$ is odd}\\
		\frac{p_1}{2}+\sum_{i= p_1/2}^{p_1-1}p_1 +1\quad &\textnormal{if $p_1$ is even}
	\end{cases}
\end{equation*}
Furthermore we provide the following three basic ideas that a formal proof could follow. These ideas can essentially also be found in our proof of Proposition~\ref{prop:boundConditionInputOne}.
\begin{enumerate}
	\item This step remains to be shown: For an arbitrary arrangement of hyperplanes one can move the hyperplanes such that the histogram is increased with respect to $\preceq$ and such that a hot region appears, i.e. a region that is on the positive side of all hyperplanes. In particular one only needs to consider oriented hyperplane configurations inducing a hot region.
	\item This step needs to be formalized but is correct, see Figure~\ref{fig:conjecturenotes}: For an arrangement of hyperplanes that induces a hot region one can shift the hyperplanes that do not form a part of the hot region's boundary towards the hot region until they they are part of the hot region's boundary. This process can only increase the activation histogram. In particular we only need to consider configurations in general position with a hot region such that every hyperplane is part of its boundary.
	\item This step also needs to be formalized but is correct, see Figure~\ref{fig:conjecturenotes}: Among all arrangements with a hot region such that all hyperplanes form a part of its boundary, the corresponding activation histogram is larger if this region is located at the center. To prove this formally, one could introduce a distance between regions defined by the smallest number of region boundaries that need to be crossed for a connecting path. The number of active neurons for an arbitrary region is then $n_1$ minus the distance of this region to the hot region such that it is beneficial that the hot region is located at the center to avoid a large distance to other regions.
\end{enumerate}

\begin{figure}
  \centering
  \includegraphics[width=\textwidth]{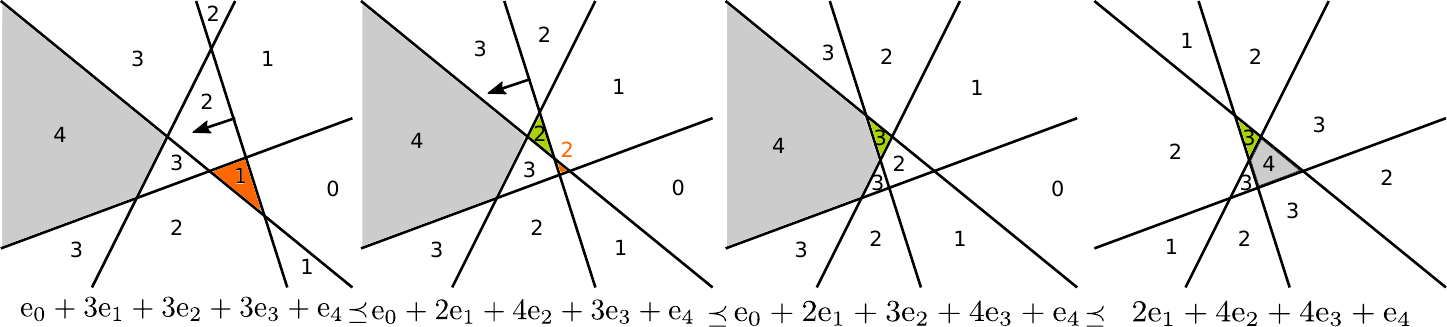}
  \caption{Illustration of our conjecture motivation indicating the hot region with gray colour with corresponding activation histograms below. The first three configurations demonstrate the second idea in the motivation. When there is a hot region, hyperplanes not defining its boundary can be moved towards the hot region while increasing the histogram. The right-most configuration illustrates the third idea and places the hot region in the center to further increase the histogram.}
\label{fig:conjecturenotes}
\end{figure}
The above three steps can be used for a proof with an arbitrary dimensions $p_0$ and indicate that the activation histogram of every configuration of oriented hyperplanes is $\preceq$-dominated by a configuration with a hot center region such that all hyperplanes form a part of its boundary. For such a configuration, a concrete formula for the corresponding hyperplane arrangement has to be found. Our Conjecture~\ref{conj:two} for the case $p_0=2$ is directly derived from this reasoning.
\subsubsection{Recursive Property}
The activation histogram joins satisfy an interesting recursive bound that might provide some fruitful insight into the problem. It is based on an analysis of how an additional hyperplane that is added to an oriented existing hyperplane arrangement affects the attained activation patterns of that arrangement. More precisely, the additional hyperplane divides some of the previous regions into two and for all regions that lie on the active side of the new hyperplane the number of ones in the activation pattern, i.e. number of active neurons is increased by one. To reflect this change, we introduce the \emph{shift operator} $\pi$ on $V$ by
\begin{equation}
	\label{eq:shift}
	\pi: V\to V, v\mapsto\sum_{i=0}^\infty v_i{\rm e}_{i+1},
\end{equation}Intuitively, it shifts all entries down to the next index as depicted in Figure~\ref{fig:shift}. Note that obviously
	 \begin{equation}
	 \label{eq:Pipreceq}
	 v\preceq v'\implies  \pi(v)\preceq \pi(v') \textnormal{ for } v,v'\in V.
	 \end{equation}\begin{figure}
  \centering
  \includegraphics[width=0.3\textwidth]{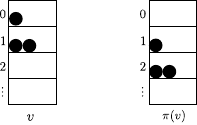}
  \caption{Visualization the shift operator $\pi$ applied to $v=e_0+2e_1$.}
\label{fig:shift}
\end{figure}We can now precisely express the recursive relation as follows. The proof is deferred to Appendix~\ref{sec:recursionProperty}.
\begin{proposition}
	\label{prop:recursion}
	For $p_1\ge p_0\ge 1$ it holds that
	$\tau_{p_0+1}^{p_1+1}\preceq \pi(\tau_{p_0+1}^{p_1})+\tau_{p_0}^{p_1}$.
\end{proposition}
We will use this property in next Section~\ref{sec:improvedBound} to derive tighter bounds using Theorem~\ref{thm:mainFirstOrder}. Figure~\ref{fig:tau} summarizes our results on the activation histogram join. To the best of our knowledge, an explicit formula of the activation histogram join for $p_1>p_0>1$ is not yet discovered apart from our conjecture for $p_0=2$.
\begin{figure}
  \centering
  \includegraphics[width=0.3\textwidth]{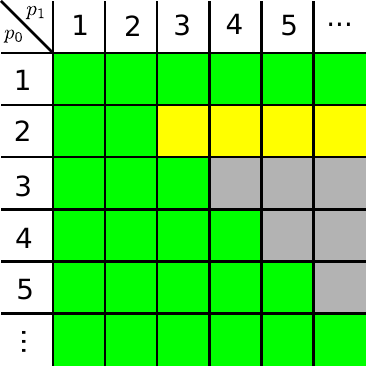}
  \caption{Indices $(p_0,p_1)$ for which the activation histogram join $\tau_{p_0}^{p_1}$ is known by Proposition~\ref{prop:boundConditionInputOne} and Lemma~\ref{lem:explicitTauPowerset} (green), for which the conjecture applies (yellow) and other indices (gray).}
  \label{fig:tau}
\end{figure}
	\subsubsection{A Note on Scientific Priority}
	When we first derived Propositions~\ref{prop:boundConditionInputOne} and ~\ref{prop:recursion} we refrained from publication because on its own and without the context presented in this work they seem marginal. It appears that the essential ingredients for the proofs were independently discovered then published in the preprint~\cite{xie2020general} by Xie et al. However, we present our own proofs for these results in this work.

	Furthermore, in a discussion with T. Zaslavsky, he pointed out that the upper bound for the number of regions of a hyperplane arrangement \eqref{eq:schlaefliBound} which is often mistakenly attributed to him in the literature on region-counting (\cite{Montufar17}, \cite{DBLP:BoundingCounting}, \cite{hinz}, \cite{pmlr-v119-xiong20a}), is actually a result of Ludwig Schläfli from his \emph{Theorie der vielfachen Kontinuität} written in 1850-1852, see his republished work~\cite{Schlafli1950}. In 1943, R.C. Buck provided a different proof, see \cite{buck1943partition}.
 \subsection{Unfolding the recursion}
 \label{sec:improvedBound}
 In this section we present a new improved collection of elementary bounds $\gamma^*\in\Gamma$ that is tighter than the presented collections $\widehat{\gamma},\tilde\gamma, \bar{\gamma}_{n,p_1}$ from equation \eqref{eq:oldGammas}. Using Theorem~\ref{thm:mainFirstOrder}, we obtain slightly stronger resulting composed bounds for the number of regions.
The idea for the construction of $\gamma^*$ is based on our results on the activation histogram join $\tau$ from Section~\ref{sec:activationHistogramProblem}.
 \begin{definition}
	 \label{def:improvedCollection}
	 For $p_1\in\mathbb{N}_{+}$ and let
		 $\gamma^*_{1,p_1}=\tau_{1}^{p_1}$
	 Furthermore, for $p_0\in\left\{ 2,\dots,p_1 \right\}$ define recursively
	 \begin{equation}
		 \label{eq:gammaStarRecursion}
		 \gamma^*_{p_0,p_1}=\pi\left( \gamma^*_{\min(p_0,p_1-1),p_1-1} \right)+\gamma^*_{p_0-1,p_1-1}
	 \end{equation}
 \end{definition}
 The anchor definition of $\gamma^*_{1,p_1}$ for $p_1\in\mathbb{N}_+$ is explicitly known by Proposition~\ref{prop:boundConditionInputOne} and the recursion in equation~\eqref{eq:gammaStarRecursion} corresponds to Proposition~\ref{prop:recursion}. This recursive definition can be explicitly unfolded.
	\begin{proposition}
		\label{prop:improvedExplicit}
		For $p_1\ge p_0\ge 2$, it holds that
	\begin{align*}
		\gamma^*_{p_0,p_1}=&\; \mathds{1}_\mathbb{N}(\tfrac{p_1-p_0}{2}){\rm e}_{ \frac{p_1-p_0}{2}} + \sum_{k=\lfloor \frac{p_1-p_0}{2}\rfloor+1}^{p_1-p_0}\left( \tbinom{2p_0+2k-p_1-2}{p_0-1}+\tbinom{2p_0+2k-p_1-1}{p_0-1} \right){\rm e}_{k}+\sum_{k=p_1-p_0+1}^{p_1}\tbinom{p_1}{ p_1-k}{\rm e}_k
	\end{align*}
	with the indicator function satisfying $\mathds{1}_{\mathbb{N}}(t)=1$ for $t\in\mathbb{N}$ and $\mathds{1}_{\mathbb{N}}(t)=0$ otherwise.
	\end{proposition}
	The next lemma shows that our proposed collection $\gamma^*$ is indeed an element of $\Gamma$. 
 \begin{lemma}
	 The collection $\gamma^*$ satisfies the bound condition from Definition~\ref{def:boundCondition}.
	 \label{lem:gammastar}
 \end{lemma}

	\begin{table}
  \centering
  {\def\arraystretch{1.3}\tabcolsep=3pt
  \begin{tabular}[htb]{|l||c|c|c|c|c|c|c|c}
	\hline
	\backslashbox{i}{$p_0$}&0&1&2&3&4&5&6 \\
	\hline\hline
$0$&$0$&$0$&$0$&$0$&$0$&$0$&$1$\\
$1$&$0$&$0$&$0$&$0$&$1$&$6$&$6$\\
$2$&$0$&$0$&$1$&$4$&$14$&$15$&$15$\\
$3$&$0$&$2$&$5$&$16$&$20$&$20$&$20$\\
$4$&$0$&$2$&$9$&$15$&$15$&$15$&$15$\\
$5$&$0$&$2$&$6$&$6$&$6$&$6$&$6$\\
$6$&$1$&$1$&$1$&$1$&$1$&$1$&$1$\\
$7$&$0$&$0$&$0$&$0$&$0$&${0}$&$0$\\
	$\vdots$&$0$&$0$&$0$&$0$&$0$&$0$&$0$\\
  \end{tabular}
}
\caption{The values of $(\gamma^*_{p_0,6})_i$ and different values for $i$ and $p_0$. For example every activation histograms $\sum_{s\in \mathcal{S}_h}{\rm e}_{|s|}$ with of a layer transition function $h\in\mathbb{R}^{3}\to\mathbb{R}^6$ is bounded by $\gamma^*_{3,6}=4{\rm e}_2+16{\rm e}_3+15{\rm e}_4+6{\rm e}_5+{\rm e}_6$. In particular every arrangement of $6$ oriented hyperplanes in $\mathbb{R}^{3}$ induces at most $38=1+6+15+16$ regions that are on the active side of at least $3$ of those regions.}
  \label{tab:recursion6}
	\end{table}
	Therefore, we can use $\gamma^*$ in Theorem~\ref{thm:mainFirstOrder}. The following lemma compares our new collection $\gamma^*$ with those from equation~\eqref{eq:oldGammas}.
 \begin{lemma}
	 \label{lem:newtighter}
	 The collection $ \gamma^*$ from Definition~\ref{def:improvedCollection} is at least as tight as the collection $\bar{\gamma}$ from equation~\eqref{eq:oldGammas}, i.e. for any $p_1\in\mathbb{N}_+$ and $p_0\in\left\{ 0,\dots,p_1 \right\}$, it holds that $\gamma^*_{p_0,p_1}\preceq\bar{\gamma}_{p_0,p_1}$.
 \end{lemma}
 In particular, for all input and output dimension $p_0,p_1$ whenever the explicit formulas of $\gamma^*_{p_0,p_1}$ and $\bar{\gamma}_{p_0,p_1}$ from Proposition~\ref{prop:improvedExplicit} and equation~\eqref{eq:oldGammas} respectively do not match, our new collection is indeed strictly tighter. The following example demonstrates how our results can be used to compute concrete bounds on the number of regions using Theorem~\ref{thm:mainFirstOrder}.
 \begin{example}
	 Consider output dimension $p_1=6$. The entries of $(\gamma^*_{p_0,6})_i$ as given by Proposition~\ref{prop:improvedExplicit} are printed in Table \ref{tab:recursion6} for $p_0\in\{0,\ldots,6\}$ and $i\in \mathbb{N}$. To each column with input dimension $p_0$, the clipping function $\textnormal{cl}_{p_0}$ has to be applied to derive the matrices used in the matrix version of the framework's bound as given in Theorem~\ref{thm:mainFirstOrder}. They are denoted by $\bar{B}_6$ and $B^*_6$ for $\bar\gamma$ and $\gamma^*$ respectively. We also derived the matrix $B_6^{\text{con}}$ that is induced similarly to Definition when our conjecture is used for $p_0=2$ and the recursion for $p_0>2$. These matrices are given by
	\begin{equation*}
	\bar{B}_6=
	\tiny
	\begin{pmatrix}
		1&  0&   0&   0&   0&   0&   1\\
		0&  7&   0&   0&   0&   6&   6\\
		0&  0&  22&   0&  15&  15&  15\\
		0&  0&   0&  42&  20&  20&  20\\
		0&  0&   0&   0&  22&  15&  15\\
		0&  0&   0&   0&   0&   7&   6\\
		0&  0&   0&   0&   0&   0&   1\\
	\end{pmatrix}
	B^*_6=
	\tiny
	\begin{pmatrix}
		1&  0&   0&   0&   0&   0&   1\\
		0&  7&   0&   0&   1&   6&   6\\
		0&  0&  22&   4&  14&  15&  15\\
		0&  0&   0&  38&  20&  20&  20\\
		0&  0&   0&   0&  22&  15&  15\\
		0&  0&   0&   0&   0&   7&   6\\
		0&  0&   0&   0&   0&   0&   1\\
	\end{pmatrix}
	B_6^{\text{con}}=
	\begin{pmatrix}
		1&  0&   0&   0&   0&   0&   1\\
		0&  7&   0&   0&   2&   6&   6\\
		0&  0&  22&   7&  13&  15&  15\\
		0&  0&   0&  35&  20&  20&  20\\
		0&  0&   0&   0&  22&  15&  15\\
		0&  0&   0&   0&   0&   7&   6\\
		0&  0&   0&   0&   0&   0&   1\\
	\end{pmatrix}.
	\end{equation*}
	Note that the maximal eigenvalue of $\bar{B}_{6}$ is $\sum_{j=0}^3\tbinom{6}{ j}=42$ whereas the maximal eigenvalue of $B^*_6$ is $38$. It follows that for a neural network with $L$ layers of equal width $n_1=\dots=n_L=6$ and arbitrary input width $n_0\in\mathbb{N}$ the old collection $\bar\gamma$ provides a bound on the number of affine linear regions of order $\mathcal{O}(42^L)$, whereas the our new collection $\gamma^*$ improves this to $\mathcal{O}(38^L)$. If our conjecture for $p_0=2$ is true this can be cut down to $\mathcal{O}(35^{L})$. Further improvements can be achieved by directly using an explicit formula for the histogram join to avoid the use of the non-tight recursion.
 \end{example}

\section{Composing subnetwork activation histograms}
\label{sec:higherOrder}
In this section we want to generalize the theory in~\cite{hinz}, specifically Theorem~\ref{thm:mainFirstOrder}, to allow for the composition of worst-case bounds on the activation histogram induced by subnetworks, i.e. multiple consecutive network layers instead of only one layer transition function. We first motivate why this is beneficial in Section~\ref{sec:compositionLoss} by explaining why this extension will allow to construct tighter bounds on the number of affine-linear regions. Then we present the formal result in Section~\ref{sec:formalGeneralization}.
 \subsection{Motivation: Composition Loss}
 \label{sec:compositionLoss}
 The bound on the number of regions in Theorem~\ref{thm:mainFirstOrder} is based on a too conservative worst-case analysis estimate to bound the number of regions. More precisely, it is assumed that every regions induced by the $l-1$ layers of the network is cut maximally into subregions by the next $l$-th layer transition function $h_{l}$ for $l\in\{1,\ldots,L\}$. However this cannot happen at the same time as the counter example shows. 
 \begin{example}
	 \label{ex:compositionloss}
	 Consider a layer transition function $h$ with input dimension $1$ and output dimension $3$ defined by
	 \begin{equation*}
		 h:\mathbb{R}^{1}\to\mathbb{R}^{3}, x\mapsto 
		 {
			 \scriptsize
		 \begin{pmatrix}
		 \textnormal{ReLU}(-(x-1))\\
		 \textnormal{ReLU}(x)\\
		 \textnormal{ReLU}(x-2)
		 \end{pmatrix}
	 }
	 \end{equation*}
	 The $\textnormal{ReLU}$ activation function used in all 3 coordinates of $h$ defines the three oriented hyperplanes $H_1={1}$, $H_2={0}$ and $H_3={2}$ in $\mathbb{R}^{1}$ see Figure~\ref{fig:compositionGap} (orientation indicated by ``$+$'' and ``$-$''). They partition the space $\mathbb{R}$ into 4 regions. The image of these regions in $\mathbb{R}^3$ is a line with three kinks indicated on the right hand side. It is impossible for a hyperplane in $\mathbb{R}^{3}$ to intersect all 4 sections of different affine-linear behaviour, whereas the framework~\cite{hinz} treats each of those sections as if they were the whole space $\mathbb{R}^1$ and could be maximally intersected with an individual set of hyperplanes.
 \end{example}
\begin{figure}[htbp]
  \centering
  \includegraphics[width=0.7\textwidth]{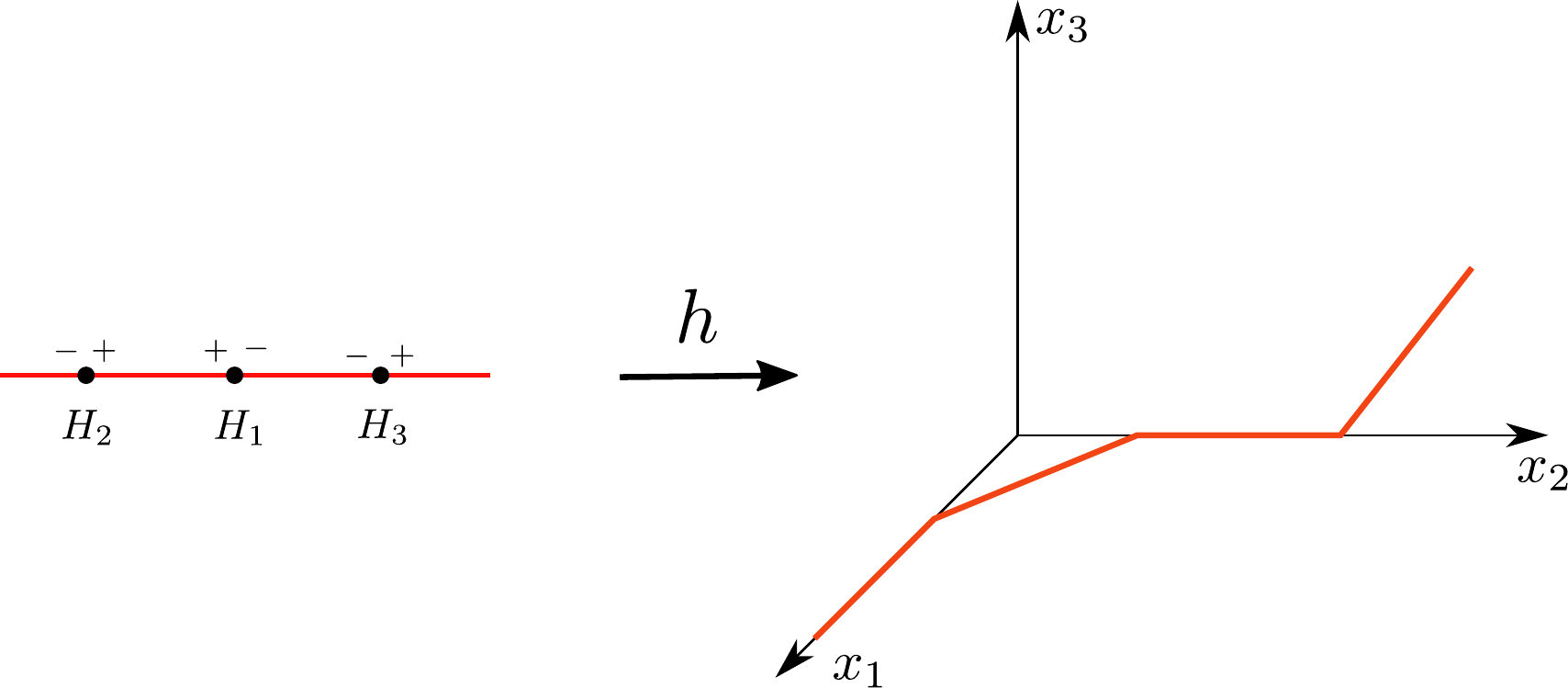}
  \caption{Illustration of the image of $h$ in $\mathbb{R}^{3}$ with three kinks separating four line sections. No hyperplane in $\mathbb{R}^{3}$ can intersect with all four line sections.}
  \label{fig:compositionGap}
\end{figure}
In particular, every of the compositions of collections of activation histogram bounds for the individual layers in Theorem~\ref{thm:mainFirstOrder} introduces a loss in tightness which we call the \emph{composition loss}. To reduce the number of compositions necessary, we generalize our framework to accept activation histogram bounds corresponding to whole blocks of consecutive network layers instead of just one.
 
\subsection{Formal statement}
\label{sec:formalGeneralization}
\subsubsection{Preliminaries}

The framework~\cite{hinz} and hence all derived bounds for the number of affine regions of feed-forward neural networks are based on the idea to bound the activation histograms when moving from one layer to the next. Bounds for neural networks are then composed of these layer-wise bounds. In this section we will develop a theory that allows to compose subnetworks for which we need appropriate formalism. We first generalize the activation histogram join.
\begin{definition}
	\label{def:multitau}
	For $p_0\in\mathbb{N}$, $l\in\mathbb{N}_+$ and $\mathbf{p}\in\mathbb{N}^l$ let
	\begin{equation*}
	\tau_{p_0}^{\mathbf{p}}:=\bigvee_{\mathbf{h}\in \textnormal{RL}(p_0,\mathbf{p})}\sum_{s\in\mathcal{S}_\mathbf{h}}{\rm e}_{\min(|s_1|,\dots,|s_l|)}
	\end{equation*}
\end{definition}
For $(h_1,\dots,h_l)=\mathbf{h}\in\textnormal{RL}(p_0,\mathbf{p})$ and a neural network $f_{\mathbf{h}}:=h_p\circ\dots\circ h_1:\mathbb{R}^{n_0}\to\mathbb{R}^{n_l}$ defined as in equation~\eqref{eq:neuralNet}, the sum $\sum_{s\in\mathcal{S}_\mathbf{h}}{\rm e}_{\min(|s_1|,\dots,|s_l|)}$ is the histogram of the minimum of the number of active neurons in the in the individual layers over all occurring regions. The join $\tau_{p_0}^{\mathbf{p}}$ is the smallest upper bound for all such histograms that can occur for $\mathbf{h}\in\textnormal{RL}(p_0,\mathbf{p})$. 
In Definition~\ref{def:boundCondition} we defined the layer-wise bound condition for collections $\gamma_{p_0,p_1}$ with $p_1\in\mathbb{N}$ and $p_0\in\left\{ 0,\dots,p_1 \right\}$. We need to provide a similar definition adapted to groups multiple layers identified by their layer widths which we call will \emph{network topology}.
\begin{definition}Let $m\in\mathbb{N}_+$ and $\mathbf{p}=(p_1,\dots,p_m)\in\mathbb{N}^m_+$. A collection of elements $(\gamma^{\mathbf{p}}_{p_0})_{p_0\in\left\{ 0,\dots,p_1 \right\}}$ in $V$ satisfies the \emph{subnetwork bound condition of topology} $\mathbf{p}$ if the following is true: 
	\label{def:generalBoundCondition}
	\begin{enumerate}
		\item $\forall p_0\in \left\{ 0,\dots,p_1 \right\}:\quad \tau^{\mathbf{p}}_{p_0}\preceq \gamma_{p_0}^{\mathbf{p}}$
		\item $\forall p_0,\tilde{p}_0\in \left\{ 0,\dots,p_1 \right\}\quad p_0\le\tilde{p}_0\implies \gamma_{p_0}^{\mathbf{p}}\preceq\gamma_{\tilde{p}_0}^{\mathbf{p}}$.
	\end{enumerate}
	Furthermore, we define the set 
  \begin{equation*}
	  \Gamma^{\mathbf{p}}:=\left\{ (\gamma_{p_0}^{\mathbf{p}})_{p_0\in\left\{ 0,\dots,p_1 \right\}}\mid \gamma \text{ satisfies the subnetwork bound condition of topology }\mathbf{p} \right\}.
  \end{equation*}
\end{definition}
Note that the uppercase $\mathbf{p}$ in $\gamma^{\mathbf{p}}_{p_0}$ above is solely for notational convenience to indicate that the topology $\mathbf{p}$ is meant. For $\mathbf{p}=(p_1,\ldots,p_m)\in\mathbb{N}_+^m$ and $\gamma^{\mathbf{p}}\in\Gamma^{\mathbf{p}}$ we define the generalization of the transformation rules from equation~\eqref{eq:transformationRule} by
   \begin{equation}
	 \label{eq:generalPhiFunc}
	   \varphi^{(\gamma^\mathbf{p})}: V\to V,\;
   v\mapsto \sum_{i=0}^{\infty}v_i \textnormal{cl}_{\min(i,p_1)}(\gamma_{\min(i,p_1)}^{\mathbf{p}}).
   \end{equation}
   Note that this sum only involves finitely many non-zero terms by the definition of $V$.
   \subsubsection{Main Result}
   \label{sec:mainResult}
\begin{figure}[htbp]
  \centering
  \includegraphics[width=1.0\textwidth]{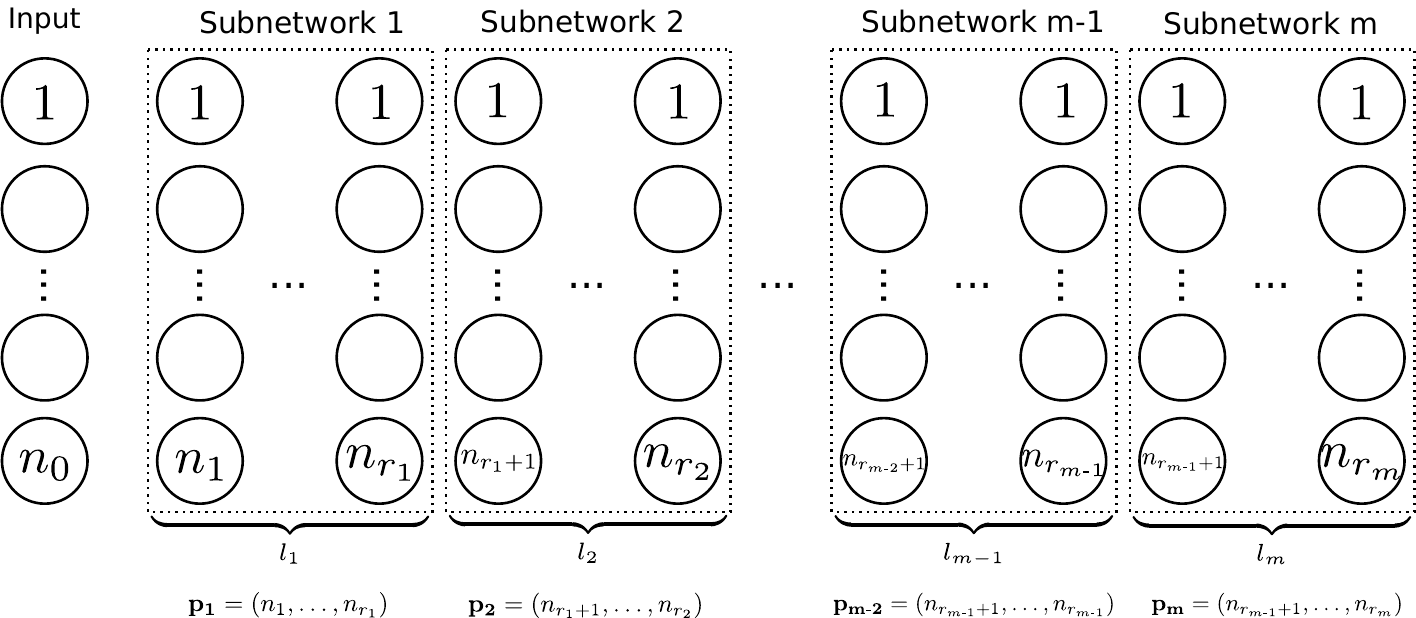}
  \caption{Partition of the feed-forward neural net $f$ into subnetworks.}
  \label{fig:partition}
\end{figure}

   In order to express our main result on the composition of higher order elementary bounds, we need to partition consecutive layers into a number $m\in\mathbb{N}$ of subnetworks. The layer index boundaries separating these subnetworks are denoted by $r_0,\dots,r_{m}\in\mathbb{N}$ with 
   \begin{equation*}
	   r_0:=0<r_1<\dots<r_{m-1}<r_m:=L.
   \end{equation*}
   For $i\in\left\{ 1,\dots,m \right\}$ we now have $l_{i}:=r_{i}-r_{i-1}$ layer transitions in the $i$-th subnetwork with widths $\mathbf{p}_i=(n_{r_{i-1}+1},\dots,n_{r_i})$, see Figure~\ref{fig:partition}. Below we present the generalization of the main result of the framework~\cite{hinz} to bound the number of regions if this network based on the partition into subnetworks.

   \begin{theorem}Assume $m,r_0,\dots,r_m,\mathbf{p}_1,\dots\mathbf{p}_m$ are chosen as above and let $\gamma^{\mathbf{p}_i}\in\Gamma^{\mathbf{p}_i}$ for $i\in\left\{ 1,\dots,m \right\}$. The number of attained activation patterns $|\mathcal{S}_{\mathbf{h}}|$ is bounded by 
   \label{thm:generalMainResultPhi}
 \begin{equation}
	 |\mathcal{S}_{\mathbf{h}}|\le\|\varphi^{(\gamma^{\mathbf{p}_m})}\circ\dots\circ\varphi^{(\gamma^{\mathbf{p}_1})}({\rm e}_{n_0})\|_1.
   \label{eq:generalMainRecursivePhi}
 \end{equation}
 Equivalently, in matrix formulation it holds that
	\begin{equation}
	 \label{eq:generalMainRecursiveMatrix}
	 |\mathcal{S}_{\mathbf{h}}|\le \|B^{(\gamma^{\mathbf{p}_m})}M_{n_{r_{m-2}+1},n_{r_{m-1}+1}}\dots 
	 B^{(\gamma^{\mathbf{p}_2})}
	 M_{n_{r_{0}+1},n_{r_{1}+1}}
		 B^{(\gamma^{\mathbf{p}_1})}M_{n_{0},n_{r_0+1}}e_{n_0+1}\|_1
	 \end{equation}
	with $M_{n,p_1}$, $e_{n_0+1}$  defined by equations~\eqref{eq:M}, \eqref{eq:e} for $n,p_1\in\mathbb{N}_+$ and square matrices $B^{(\gamma^{\mathbf{p}_k})}\in\mathbb{N}^{(n_{r_{k-1}+1}+1)\times (n_{r_{k-1}+1}+1)}$ defined by $(B^{(\gamma^{\mathbf{p}_k})})_{i,j}= \left(\varphi^{(\gamma^{\mathbf{p}_k})}({\rm e}_{j-1})\right)_{i-1}=\left( \textnormal{cl}_{j-1}(\gamma^{\mathbf{p}_k}_{j-1}) \right)_{i-1}$ for $i,j\in\left\{ 1,\dots,n_{r_{k-1}+1}+1 \right\}, k\in\left\{ 1,\dots,m \right\}$.
 \end{theorem}
 To make use of this theorem one needs $m$ collections of histograms satisfying the subnetwork bound bound condition for the topologies $\mathbf{p}_1,\ldots,\mathbf{p}_m$ of the subnetworks. The above Example~\ref{ex:compositionloss} demonstrated that in principle, it is possible to achieve better resulting bounds on the number of regions than can be obtained from Theorem~\ref{thm:mainFirstOrder}. However, this requires sufficiently tight histogram bound collections $\gamma^{\mathbf{p}_1},\ldots,\gamma^{\mathbf{p}_m}$ that are not simply based on layer-wise considerations such that new concepts or ideas are required to take advantage of Theorem~\ref{thm:generalMainResultPhi}. 
 \section{Summary}
 \label{sec:summary}
 The work ~\cite{hinz} introduced a general framework to derive upper bounds on the affine number of regions in feed-forward ReLU neural networks based on a layer-wise worst-case analysis on activation histograms. In this work we have extended these results in two ways. Firstly, we have elaborated in detail how this framework can be fully exploited and secondly we have generalized it to allow subnetwork instead of layer-wise activation histogram bounds. 

 More precisely, our first contribution consists of a precise analysis of the activation histogram join for a ReLU layer transition function $h:\mathbb{R}^{p_0}\to\mathbb{R}^{p_1}$, $p_0,p_1\in\mathbb{R}$ which maps the output of one layer to the next. The $p_1$ hyperplanes in $\mathbb{R}^{n_0}$ induced implicitly by the $\textnormal{ReLU}$ activation function and the weight and bias parameters of $h$ partition the space $\mathbb{R}^{n_0}$ in a number of regions on which $h$ is affine. Not only the number of these regions but also the histogram of how many neurons are active on these regions is of relevance for the framework because the number of active neurons bounds the rank of $h$ on a region. It is therefore of interest to find the lowest worst case bound (join) of these histograms with respect to an appropriate order relation on the histograms. 

 We solved this activation histogram join problem for input dimension $p_0=1$ and arbitrary output dimension $p_1\in\mathbb{N}$, derived a conjecture for the case $p_0=2, p_1\in\mathbb{N}$ and motivated steps for a proof in higher dimensions $p_0>2$ and $p_1\in\mathbb{N}$. Furthermore, geometrical considerations lead to a recursive histogram bound which allows us to conclude tighter bounds on the number of regions than previously presented in~\cite{hinz} by the use of a recursive definition starting with our solution for $p_0=1$ or our conjecture for $p_0=2$. For the former case we gave an explicit formula by unfolding the recursive definition. For an explicit solution for the activation histogram join for all $p_0,p_1\in\mathbb{N}$ the framework~\cite{hinz} would provide even tighter results. Table~\ref{tab:evolution} gives an overview on the evolution of bounds on the number of regions for fully connected networks ordered from weak to strong.

 \begin{table}[htpb]
 	\centering
 	\begin{tabular}{|c|c|c|c|}
		\hline
		Bound, Reference&
		\begin{tabular}{@{}c@{}}
			Information\\ carrier
		\end{tabular}
		&Layer-wise bound based on
		\\
		\hline
		\hline
		$2^{\textnormal{\# neorons}}$&number&
		\begin{tabular}{@{}c@{}}
			set theory,\\ 
			each neuron doubles the bound
		\end{tabular}\\
		\hline
		\begin{tabular}{@{}c@{}}
		$\prod_{l=1}^L
		\sum_{i=1}^{\min(n_0,\ldots,n_{l-1})}
		\tbinom{n_l}{i}$\\
		from \cite{Montufar:2014:NLR:2969033.2969153}, 2014
		\end{tabular}
		&number&L. Schläfli's bound~\eqref{eq:schlaefliBound}\\
		\hline
		\begin{tabular}{@{}c@{}}
		Theorem~\ref{thm:mainFirstOrder} using $\bar\gamma$\\
		from \cite{hinz} (partly \cite{DBLP:BoundingCounting}), 2019
		\end{tabular}&histogram&
		\begin{tabular}{@{}c@{}}L. Schläfli's bound + \\ activation combinatorics\end{tabular}\\
		\hline
		\begin{tabular}{@{}c@{}}
			Theorem~\ref{thm:mainFirstOrder} using $\gamma^*$ (or our conjecture)\\
			from this work (partly \cite{xie2020general}), 2021
		\end{tabular}&histogram&
		\begin{tabular}{@{}c@{}}activation histogram join for\\ input dimension 1 (or 2)+recursion\end{tabular}
		 \\
		\hline
		not yet discovered&histogram&activation histogram join\\
		\hline
 	\end{tabular}
	\caption{Evolution of upper bounds on the number of regions based on layer-wise activation histogram bounds. The use of an explicit solution for the activation histogram join will provide the tightest bound obtainable from the framework~\cite{hinz}.}
	\label{tab:evolution}
 \end{table}

 As a byproduct our of our conjecture and considerations on how to upper bound the activation histogram join we observed that oriented hyperplane arrangements inducing a hot center, i.e. a region that is located in the center of all regions and on the active side of all hyperplanes tends to generate large activation histograms corresponding to regions on which many neurons are active whereas the opposite, i.e. a cold center where no neuron is active generates small activation histograms corresponding to regions where only a few neurons are active. This insight might be relevant for parameter initialization because it indicates on how the number of regions can be controlled: Parameter configurations where every layer transition function induces a cold center might induce fewer regions than configurations where all layer-wise mappings induce a hot center.

 We also generalized the framework~\cite{hinz} to allow the composition of subnetwork instead of only layer-wise activation histogram bounds. This is important to reduce the number of compositions necessary to represent the whole network because every such composition introduces a loss of tightness in the final bound on the number of regions. In particular, as soon as appropriate activation histogram bounds for subnetworks are developed our framework extension will provide of the theoretical foundation for even tighter upper bounds on the number of regions in feed-forward ReLU neural networks.

	 Finally we want to thank Prof. Tom Zaslavsky\footnote{Binghamton University, New York} who provided helpful comments, especially on the activation histogram join in Section~\ref{sec:activationHistogramProblem} and on the origin of the bound~\eqref{eq:schlaefliBound}.
  \newpage
  \appendix

  \section{Results for the composition of layer-wise histogram bounds}
  \subsection{Results on the activation histogram join}
  \label{sec:signatureJoinResults}
 \begin{proof}[Proof of Lemma~\ref{lem:tauBoundCondition}]
		 The first property of the bound condition is obviously satisfied. For the second property let $p_0,p_1\in\mathbb{N}_+$ with $p_0<p_1$ and $h\in\textnormal{RL}(p_0,p_1)$ with corresponding weight matrix $W_h\in\mathbb{R}^{p_1\times p_0}$ and bias vector $b_h\in\mathbb{R}^{p_1}$. If we extend the matrix $W_h$ to $\tilde{W}_h\in\mathbb{R}^{p_1\times p_0+1}$ with an arbitrary additional $(p_0+1)$-th column, the layer transition function $\tilde h:\mathbb{R}^{p_0+1}\to\mathbb{R}^{p_1}, x\mapsto \textnormal{ReLU}(\tilde{W}_hx+b_h)$ with coordinate-wise application of the activation function satisfies
		 \begin{equation*}
			 \mathcal{S}_{\tilde h}\supseteq \left\{ S_{\tilde h}(x)\mid x\in \mathbb{R}^{p_0}\times \left\{ 0 \right\} \right\}=
			 \left\{ S_{h}(x)\mid x\in \mathbb{R}^{p_0} \right\}=\mathcal{S}_h,
	 \end{equation*} from which it follows that $\tau_{p_0}^{p_1}\preceq\tau_{p_0+1}^{p_1}$.
	 \end{proof}
	\subsubsection{Input dimension not smaller than output dimension}
	\begin{proof}[Proof of Lemma~\ref{lem:explicitTauPowerset}]
			First $p_0> p_1$, there exist hyperplane arrangements of $p_1$ hyperplanes in $\mathbb{R}^{p_0}$ with $2^{p_1}$ induced regions, i.e. $\mathcal{S}_{h'}=\left\{ 0,1 \right\}^{p_1}$ for some $h'\in\textnormal{RL}(p_0,p_1)$. It follows that
				$\tau_{p_0}^{p_1}= \bigvee_{h'\in \textnormal{RL}(p_0,p_1)}\sum_{s\in\mathcal{S}_{h'}}{\rm e}_{|s|}= 
				\sum_{s\in\left\{ 0,1 \right\}^{p_1}}{\rm e}_{|s|}=\sum_{i=0}^{p_1}\tbinom{p_1}{i}{\rm e}_i$.
		\end{proof}
 \subsubsection{Input dimension one}
 \label{sec:inputDimOne}
 The following result shows that we only need to consider hyperplanes in general position (points on a real line that do not lie on each other) for the activation histogram join~\eqref{eq:activationHistogramJoin}.
 \begin{lemma}
	 \label{lem:restrictedMax}
 For all $p_1\in\mathbb{N}_{+}$, it holds that
 \begin{equation*}
 \tau_{1}^{p_1}= \bigvee_{\underset{|\mathcal{S}_h|=p_1+1}{h\in\textnormal{RL}(1,p_1)}}\sum_{s\in\mathcal{S}_h}{\rm e}_{|s|}.
 \end{equation*}
 \begin{proof}
	 Let $h\in\textnormal{RL}(1,p_1)$. It suffices to prove that there exists $h'\in\textnormal{RL}(1,p_1)$ with  $\sum_{s\in\mathcal{S}_{h}}{\rm e}_{|s|}\preceq \sum_{s\in\mathcal{S}_{h'}}{\rm e}_{|s|}$ and $\vert\mathcal{S}_{h'}\vert=p_1+1$. 
	 By definition, there exists $w,v\in\mathbb{R}^{p_1}$ such that for $t\in\mathbb{R}$, $h(t)=\textnormal{ReLU}(tw+v)$, where the activation function is applied coordinate-wise.
	 Define $X(v,w):=\left\{ -v_i/w_i\middle\vert\;i\in\left\{ 1,\dots,p_1 \right\}\land w_i\neq 0 \right\}$. If $|X(v,w)|\neq p_1$ we can construct slightly disturbed weights and biases $w',v'\in\mathbb{R}^{p_1}$ such that $|X(v',w')|=p_1$ and $\mathcal{S}_h\subsetneq \mathcal{S}_{h'}$ for the corresponding function $h':t\mapsto \textnormal{ReLU}.(tw'+v')$. In particular, there are $p_1+1$ regions separated by the $p_1$ points in $X(v',w')$, each with different activation pattern, i.e. $|\mathcal{S}_{h'}|=p_1+1$.
 \end{proof}
 \end{lemma}
 We can now prove explicit formulation of the activation histogram join for input dimension 1.
	 \begin{proof}[Proof of Proposition~\ref{prop:boundConditionInputOne}]
		 Let $p_1\in \mathbb{N}_+$ and $h\in \textnormal{RL}(1,p_1)$. By Lemma~\ref{lem:restrictedMax}, we can assume that there are $p_1$ points $t_1<\dots<t_{p_1}$ on the real line $\mathbb{R}$ where the affine behaviour of $h$ changes. Their orientation will be encoded as follows: For $i\in\left\{ 1,\dots,p_1 \right\}$ we say that $t_i$ has orientation $\sigma_i=1$ if $i$-th coordinate $s_i$ of the activation pattern $s\in\{ 0,1 \}^{p_1}$ is positive for regions that lie on the right side of $t_i$, otherwise it has orientation $\sigma_i=-1$. Now denote the activation patterns of the regions by $s_1,\dots,s_{p_1+1}$, where the $i$-th region is given by $(t_{i-1},t_i)$ for $i\in\left\{ 1,\dots,p_1+1 \right\}$, with $t_0=-\infty$, $t_{p_1+1}=\infty$. Note that the function $f:\left\{ 1,\dots,p_1+1 \right\}\to\mathbb{N}, i\mapsto |s_i|$ satisfies 
		 \begin{align}
			 f(i+1)-f(i)&= \sigma_i
			 \quad \textnormal{ for }i\in \left\{ 1,\dots,p_1 \right\}\\
			 f(1)&= \sum_{i=1}^{p_1}\delta_{\sigma_i,-1},
		 \end{align}i.e. the function $f$ is fully determined by $\sigma$, hence we will denote it by $f_\sigma$ from now on.
Note that 
\begin{equation}
	\mathcal{H}_\sigma:=\sum_{i=1}^{p_1+1}{\rm e}_{f_\sigma(i)}\in V \quad \textnormal{ for }\sigma\in\left\{ -1,1 \right\}^{p_1}
\end{equation}
is equal to the activation histogram $\sum_{s\in\mathcal{S}_{h}}{\rm e}_{|s|}=\mathcal{H}_{\sigma}$. Furthermore, for every $\sigma'\in\left\{ 0,1 \right\}^{p_1}$ there exists $h'\in\textnormal{RL}(1,p_1)$ such that $\sigma'$ is the orientation encoding corresponding to $h'$ as constructed above. In particular $\tau_{1}^{p_1}=\vee_{\sigma\in\left\{ 0,1 \right\}^{p_1}}\mathcal{H}_\sigma$. Define $\Sigma_i:=\big\{ \sigma\in\{ -1,1 \}^{p_1}\vert\;\sum_{j=1}^{p_1}\delta_{\sigma_j,1}=i \big\}$
 and let $\sigma_i^*\in \left\{ -1,1 \right\}^{p_1}$ be the orientation encoding where the first $i$ entries are $1$ and the remaining are $-1$.
\begin{figure}
  \centering
  \includegraphics[width=0.7\textwidth]{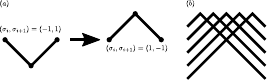}
  \caption{ (a) When there exists $i\in\left\{ 1,\ldots,p_1 \right\}$ with $(\sigma_i,\sigma_{i+1})=(-1,1)$ and we change these values to $(\sigma_i,\sigma_{i+1})=(1,-1)$ this induces a spike instead of a valley such that the corresponding histogram gets larger. (b) Among all functions $f_{\sigma^*_i}, i\in\left\{ 0,\ldots,p_1+1 \right\}$ the maximal histogram is attained when the hot region is in the center, i.e. for  $i=\lfloor p_1/2\rfloor$ or $i=\lceil p_1/2 \rceil$.}
\label{fig:sortSigma}
\end{figure}
The proof is complete if we justify every step in the following computation.
\begin{equation*}
	\tau_{1}^{p_1}=\bigvee_{\sigma\in\left\{ 0,1 \right\}^{p_1}} \mathcal{H}_\sigma\overset{(1)}{=}\bigvee_{i=0}^{p_1}\bigvee_{\sigma\in \Sigma_i} \mathcal{H}_\sigma\overset{(2)}{=}\bigvee_{i=0}^{p_1}\mathcal{H}_{\sigma_i^*}\overset{(3)}{=}\mathcal{H}_{\sigma_{\lfloor p_1/2\rfloor}^*}
\end{equation*}
Step (1) is just a partition of the joined elements into groups. For step (2), assume $i\in\left\{ 0,\dots,p_1 \right\}$ and $\sigma\in\Sigma_i$. If there exists $j\in\left\{ 1,p_1-1 \right\}$ with $\sigma_j=-1$ and $\sigma_{j+1}=1$, then $\sigma'\in\Sigma_i$ constructed from $\sigma$ by setting $\sigma'_{j}=1$ and $\sigma'_{j+1}=-1$, i.e. swapping these entries, satisfies $f_\sigma\le f_{\sigma'}$ and hence $\mathcal{H}_\sigma\preceq\mathcal{H}_{\sigma'}$, see (a) in Figure~\ref{fig:sortSigma}. In finitely many steps, one can ``move'' the ones to the left by swapping neighbouring entries and increasing the histogram while keeping the number of ones the same. When all ones are at the beginning we just have $\sigma^*_i\in\Sigma_i$, i.e. it holds that $\bigvee_{\sigma\in\Sigma_i}\mathcal{H}_{\sigma}=\mathcal{H}_{\sigma^*_i}$.
For Step (3) note that for every $i\in\left\{ 0,\dots,p_1 \right\}$, the function $f_{\sigma^*_i}$ satisfies the following:
		\begin{enumerate}
			\item The maximum is attained at $f_{\sigma^*_i}(i+1)=p_1$
			\item The function is decreasing on the right side of $i+1$: $f_{\sigma^*_i}(j+1)-f_{\sigma^*_i}(j)=-1$ for $j\in \left\{ i+1,\dots,p_1-1 \right\}$.
			\item The function is increasing on the left side of $i+1$: $f_{\sigma^*_i}(j+1)-f_{\sigma^*_i}(j)=1$ for $j\in \left\{ 1,\dots,i \right\}$.
		\end{enumerate}
		From this it follows that $\bigvee_{i=0}^{p_1}\mathcal{H}_{\sigma_i^*}=\mathcal{H}_{\sigma_{\lfloor p_1/2\rfloor}^*}=\mathcal{H}_{\sigma_{\lceil p_1/2\rceil}^*}$, see (b) in Figure~\ref{fig:sortSigma}.
The result follows from the fact that $\mathcal{H}_{\sigma_{\lfloor p_1/2\rfloor}^*}$ has the claimed form in the statement.
	 \end{proof}
	 \subsubsection{Recursion Property}
	 \label{sec:recursionProperty}
	 We first observe the following easy result about the shift operator $\pi$.
	 \begin{lemma}
		 \label{lem:preceqShift}
		 For all $v\in V$, it holds that $v\preceq \pi(v)$.
	 \end{lemma}
	 \label{sec:recursionProof}
\begin{proof}[Proof of Proposition~\ref{prop:recursion}]
	Let $p_1\in \mathbb{N}_+$, $p_0\in\mathbb{N}_+\setminus\left\{ 1 \right\}$ and $h\in\textnormal{RL}(p_0,p_1+1)$. Then $g:\mathbb{R}^{p_0}\to\mathbb{R}^{p_1}, x\mapsto (h(x)_i)_{i\in\left\{ 1,\dots,p_1 \right\}}$ is an element of $\textnormal{RL}(p_0,p_1)$.
	Now let
	\begin{equation*}
	\mathcal{S}^+=\left\{ (s_1,\dots,s_{p_1})\mid s\in\mathcal{S}_h\land s_{p_1+1}=1) \right\}, \quad
\mathcal{S}^-=\left\{ (s_1,\dots,s_{p_1})\mid s\in\mathcal{S}_h\land s_{p_1+1}=0) \right\}.
	\end{equation*}
	With this definition
	\begin{equation*}
		\sum_{s\in \mathcal{S}_{g}}{\rm e}_{|s|}=
		\sum_{s\in \mathcal{S}^+\setminus \mathcal{S}^-}{\rm e}_{|s|}+
		\sum_{s\in \mathcal{S}^+\cap \mathcal{S}^-}{\rm e}_{|s|}+
		\sum_{s\in \mathcal{S}^-\setminus \mathcal{S}^{+}}{\rm e}_{|s|}
	\end{equation*}
	and therefore
	\begin{align*}
		\sum_{s\in \mathcal{S}_{h}}{\rm e}_{|s|}
		&= 
		\sum_{s\in \mathcal{S}_{h}, s_{p_1+1}=0}{\rm e}_{|s|}+
		\sum_{s\in \mathcal{S}_{h}, s_{p_1+1}=1}{\rm e}_{|s|}
		= 
		\sum_{s\in \mathcal{S}^-}{\rm e}_{|s|}+
		\sum_{s\in \mathcal{S}^+}{\rm e}_{|s|+1}\\
		&= 
		\sum_{s\in \mathcal{S}^-\setminus \mathcal{S}^+}{\rm e}_{|s|}+
		\sum_{s\in \mathcal{S}^-\cap \mathcal{S}^+}{\rm e}_{|s|}+
		\sum_{s\in \mathcal{S}^-\cap \mathcal{S}^+}{\rm e}_{|s|+1}+
		\sum_{s\in \mathcal{S}^+\setminus \mathcal{S}^-}{\rm e}_{|s|+1}\\
		&\preceq 
		\sum_{s\in \mathcal{S}^-\cap \mathcal{S}^+}{\rm e}_{|s|}+
		\sum_{s\in \mathcal{S}^-\setminus \mathcal{S}^+}{\rm e}_{|s|+1}+
		\sum_{s\in \mathcal{S}^-\cap \mathcal{S}^+}{\rm e}_{|s|+1}+
		\sum_{s\in \mathcal{S}^+\setminus \mathcal{S}^-}{\rm e}_{|s|+1}\\
		&= 
		\sum_{s\in \mathcal{S}^+\cap \mathcal{S}^-}{\rm e}_{|s|}+
		\pi\big( \sum_{s\in \mathcal{S}_{g}}{\rm e}_{|s|} \big).
	\end{align*}
	Note that either $\sum_{s\in \mathcal{S}^+\cap \mathcal{S}^-}{\rm e}_{|s|}$ is ${\rm e}_k$ for some $k\le n_0$ which is bounded by $\tau_{p_{0}-1}^{p_1}$ or otherwise the set $U=\left\{ x\in\mathbb{R}^{p_0}\mid (W_hx+b_h)_{p_1+1}=0 \right\}$ defines a $(p_0-1)$-dimensional affine subspace of $\mathbb{R}^{n_0}$, i.e. a non-degenerated hyperplane. But in the latter case there exists a bijective linear map $\Phi:\mathbb{R}^{p_0-1}\to U$ such that
	\begin{equation*}
		\sum_{s\in \mathcal{S}^+\cap \mathcal{S}^-}{\rm e}_{|s|}=
		\sum_{s\in\left\{ S_h(x)\mid x\in U \right\}}{\rm e}_{|s|}
		=
		\sum_{s\in\left\{ S_{h\circ\Phi}(x)\mid x\in \mathbb{R}^{p_0-1} \right\}}{\rm e}_{|s|}
=
\sum_{s\in\mathcal{S}_{h\circ\Phi}}{\rm e}_{|s|}
\preceq \tau_{p_0-1}^{p_1},
	\end{equation*}where the first step follows from the fact that $\mathcal{S}^+\cap \mathcal{S}^-$ are signatures of those regions that are cut into two by the hyperplane $U$ and therefore the same signatures are attained on $U$ itself. Furthermore, 
	$\pi\big( \sum_{s\in \mathcal{S}_{g}}{\rm e}_{|s|} \big)\preceq\pi(\tau_{p_0}^{p_1})$ by equation~\eqref{eq:Pipreceq}. Since $h$ was arbitrary, this concludes the proof.
See Figure~\ref{fig:recursionProof} for an illustration.
\begin{figure}[htbp]
  \centering
  \includegraphics[width=\textwidth]{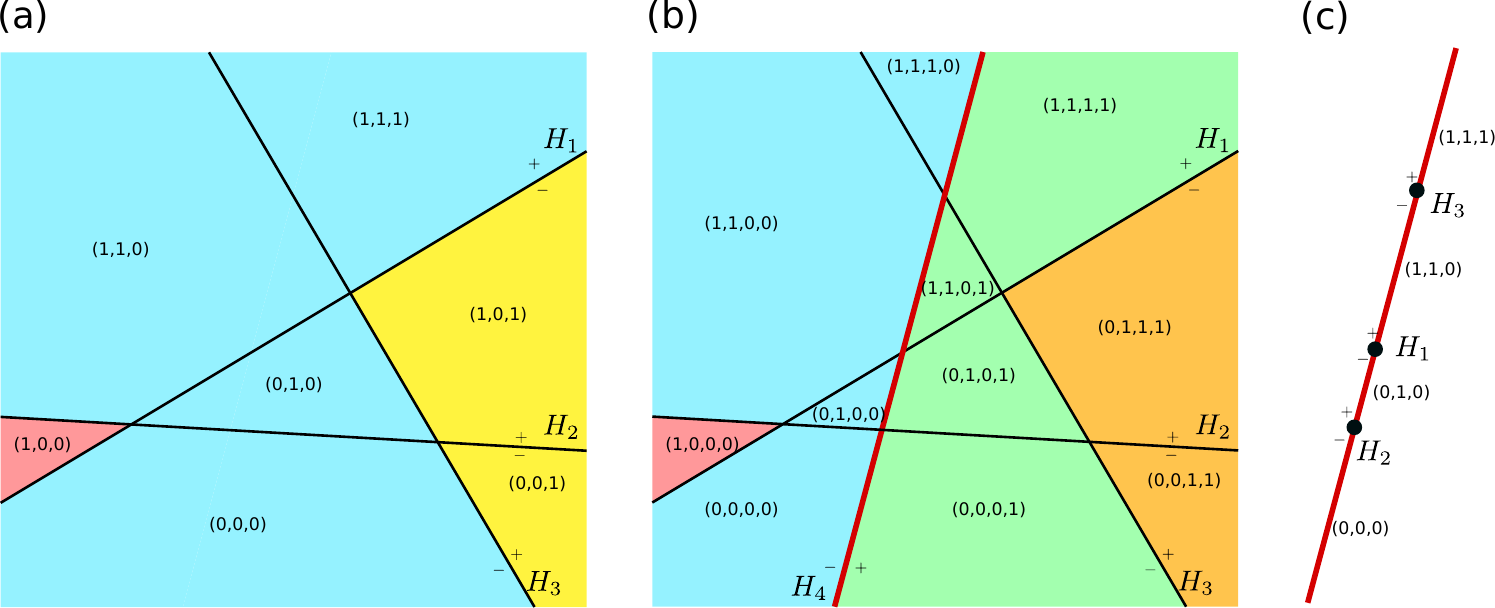}
  \caption{Illustration of the proof of Proposition~\ref{prop:recursion} for the case $p_0=2$, $p_1=3$. The hyperplanes are labeled by $H_1,\dots,H_4$ along with orientation indicated by ``$+$'' and ``$-$'' signs. 
	  The tuples represent the signatures of the respective regions. In the above case $\sum_{s\in \mathcal{S}^-\setminus\mathcal{S}^+}{\rm e}_{|s|}= {\rm e}_1$ (red area), $\sum_{s\in \mathcal{S}^+\cap\mathcal{S}^-}{\rm e}_{|s|}=\sum_{k=0}^3{\rm e}_k$ (blue area), $\sum_{s\in \mathcal{S}^+\setminus\mathcal{S}^-}{\rm e}_{|s|}={\rm e}_1+{\rm e}_2$ (yellow area) as seen in subfigure (a). In subfigure (b) a new hyperplane $H_4$ is added and $\tilde v_n^-={\rm e}_1=v_n^-$ (red area), $\sum_{s\in \mathcal{S}^+\cap\mathcal{S}^-}{\rm e}_{|s|}=\sum_{k=0}^3{\rm e}_k$ (blue area), $\sum_{s\in \mathcal{S}^+\cap\mathcal{S}^-}{\rm e}_{|s|+1}=\sum_{k=1}^4{\rm e}_k$ (green area) and $\sum_{s\in \mathcal{S}^+\setminus \mathcal{S}^-}{\rm e}_{|s|+1}={\rm e}_2+{\rm e}_3$ (orange area). Subfigure (c) visualizes the relation $\sum_{s\in \mathcal{S}^+\cap\mathcal{S}^-}{\rm e}_{|s|}\preceq \tau_{1}^{3}$ because in the $p_0-1=1$-dimensional vector space of vectors on the hyperplane $H_4$, $\sum_{s\in \mathcal{S}^+\cap\mathcal{S}^-}{\rm e}_{|s|}$ is the activation histogram of the other hyperplanes (points).
  }
  \label{fig:recursionProof}
\end{figure}
	\end{proof}
  \subsection{Unfolding the recursion}
	\label{sec:explicitExpansion}
	The purpose of this section is to prove Proposition~\ref{prop:improvedExplicit} and to compare this unfolded collection $\gamma^*$ with the collection $\bar{\gamma}$. To this end, it is suitable to define $\gamma^*_{ij}$ for all $i,j\in\mathbb{N}$ by setting $\gamma^*_{i,j}=\gamma^*_{j,j}$ for $i>j$. With this convention, the recursion~\eqref{eq:gammaStarRecursion} can be expressed as
	 \begin{equation}
		 \label{eq:gammaStarRecursionTransformed}
		 \gamma^*_{n,p_1}=\pi\left( \gamma^*_{n,p_1-1} \right)+\gamma^*_{n-1,p_1-1}.
	 \end{equation}
\begin{definition}
	For two natural numbers $\Delta_j\ge\Delta_i$, let
	\begin{equation*}
		K_{\Delta_i,\Delta_j}:
		\begin{cases}
			V&\to V\\
			v&\mapsto \tbinom{\Delta_j}{\Delta_i}\pi^{\Delta_j-\Delta_i}
		\end{cases}
	\end{equation*}
\end{definition}
\begin{figure}[htbp]
  \centering
  \includegraphics[width=\textwidth]{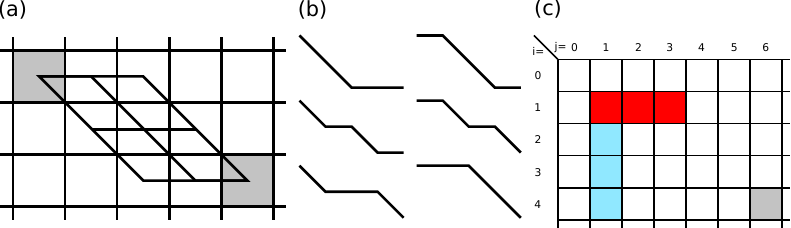}
  \caption{
	  (a) If we represent equation~\eqref{eq:gammaStarRecursionTransformed} in a two-dimensional grid, the recursive dependence is on the left and on the upper left cell. For $\Delta_i,\Delta_j\in\mathbb{N}$ the number of ways to reach a cell with index $(i+\Delta_i,j+\Delta_j)$ from a cell with index $(i,j)\in\mathbb{N}^2$  is equal to $\tbinom{\Delta_j}{\Delta_i}$. The number of horizontal steps is always $\Delta_j-\Delta_i$. In the above visualization, $\Delta_i=2$, $\Delta_j=4$ and there are $\tbinom{4}{2}=6$ possible paths depicted in subfigure (b). (c) According to the recursion~\eqref{eq:gammaStarRecursionTransformed}, the element $\gamma_{4,6}^*$ (gray) can be expressed as a function of $\gamma_{1,2},\dots,\gamma_{4,1}$ (blue entries) and $\gamma_{1,1},\dots,\gamma_{4,1}$ (red entries).
  }
  \label{fig:unfold}
\end{figure}
For $i,j\in\mathbb{N}$, $\Delta_j\ge \Delta_i$, the function $ K_{\Delta_i,\Delta_j}$ computes how $\gamma^*_{i+\Delta_i,j+\Delta_j}$ depends on $\gamma^*_{i,j}$ via the recursion~\eqref{eq:gammaStarRecursionTransformed}, see (a) in Figure~\ref{fig:unfold}. More precisely, we can expand the recursive dependence as follows.
	\begin{lemma}For $j\ge i\ge 2$, $\gamma^*_{ij}=\sum_{l=1}^{j-i+1}K_{i-2,j-l-1}(\gamma^{*}_{1,l})+\sum_{k=2}^i K_{i-k,j-1}(\gamma^*_{k,1})$.
		\label{lem:explicitAbstract}
		\begin{proof}
			The recursion~\eqref{eq:gammaStarRecursionTransformed} shows that $\gamma^*_{ij}$ can be expressed as a function of $\gamma_{2,1},\dots,\gamma_{i,1}$ and $\gamma_{1,1},\dots,\gamma_{1,j-i+1}$ (marked blue and red in (c), Figure~\ref{fig:unfold}). The dependence on the former is given by $\sum_{k=2}^i K_{i-k,j-1}(\gamma^*_{k,1})$. For the dependence on the latter, we have to take into account that every path of entries from index $(1,l)$, $l\ge 1$ to index $(i,j)$ that avoids index $(1,l+1)$ has to go through index $(2,l+1)$ which yields the term $\sum_{l=1}^{j-i+1}K_{i-2,j-l-1}(\gamma^{*}_{1,l})$ for the dependence of $\gamma^*_{i,j}$ on $\gamma^*_{1,1},\dots,\gamma^*_{1,j-i+1}$.
		\end{proof}
	\end{lemma}
	We can now prove the explicit formula:
	\begin{proof}[Proof of Proposition~\ref{prop:improvedExplicit}]
		Instead of $p_1,p_0$ we use the shorter variables $j\ge i\ge 2$. We will expand the terms from Lemma \ref{lem:explicitAbstract}
		\begin{equation}
			\gamma^*_{ij}=\sum_{k=1}^{j-i+1}K_{i-2,j-k-1}(\gamma^{*}_{1,k})+\sum_{k=2}^i K_{i-k,j-1}(\gamma^*_{k,1}).
			\label{eq:proofExplicit}
		\end{equation}
		From Definition~\ref{def:improvedCollection} and Proposition~\ref{prop:boundConditionInputOne} we know that for $k\in\mathbb{N}_{+}$, it holds that
		\begin{equation*}
			\gamma^*_{1,k}=
			\begin{cases}
				\sum_{l=0}^{\lfloor k/2\rfloor}({\rm e}_{k-l}+{\rm e}_{k-l-1})&\quad \textnormal{ if $k$ is odd}\\
				\sum_{l=0}^{\lfloor k/2\rfloor}({\rm e}_{k-l}+{\rm e}_{k-l-1})-{\rm e}_{k/2-1}&\quad \textnormal{ if $k$ is even}
			\end{cases}
		\end{equation*}
		such that
		 \begin{equation*}
			 \sum_{k=1}^{j-i+1}K_{i-2,j-k-1}(\gamma^{*}_{1,k})= \sum_{k=1}^{j-i+1}K_{i-2,j-k-1}\left( \sum_{l=0}^{\lfloor k/2\rfloor}{\rm e}_{k-l}+{\rm e}_{k-l-1} \right) -\sum_{\underset{\textnormal{$k$ even}}{k=1}}^{j-i+1}K_{i-2,j-k-1}({\rm e}_{k/2-1}).
		 \end{equation*}
		 We can expand the two terms on the right hand side:
	\begin{align*}
		\;\;& \sum_{k=1}^{j-i+1}K_{i-2,j-k-1}\left( \sum_{l=0}^{\lfloor k/2\rfloor}{\rm e}_{k-l}+{\rm e}_{k-l-1} \right) = \sum_{k=1}^{j-i+1}K_{i-2,j-k-1}\left( \sum_{l=-k}^{\lfloor k/2\rfloor-k}{\rm e}_{-l}+{\rm e}_{-l-1} \right)\\
		=& \sum_{k=1}^{j-i+1}K_{i-2,j-k-1}\left( \sum_{l=k-\lfloor k/2\rfloor}^{k}{\rm e}_{l}+{\rm e}_{l-1} \right)
		= \sum_{l=1}^{j-i+1}\sum_{k=l}^{\min(2l,j-i+1)}\tbinom{j-k-1}{ i-2}\pi^{j-i+1-k}({\rm e}_l+{\rm e}_{l-1})\\ 
		=& \sum_{l=1}^{j-i+1}\sum_{k=l}^{\min(2l,j-i+1)}\tbinom{j-k-1}{ i-2}({\rm e}_{l+j-i+1-k}+{\rm e}_{l+j-i+1-k-1}) \\
		=& \sum_{l=1}^{j-i+1}\sum_{k=-(j-i+1)}^{\min(l-(j-i+1),-l)}\tbinom{i-l-2-k}{ i-2}({\rm e}_{-k}+{\rm e}_{-k-1})\\
		=& \sum_{l=1}^{j-i+1}\sum_{k=\max(l,(j-i+1)-l)}^{(j-i+1)}\tbinom{i-2+k-l}{ i-2}({\rm e}_{k}+{\rm e}_{k-1})
		= \sum_{k=1}^{j-i+1}\sum_{l=\max(j-i+1-k,1)}^{k}\tbinom{i-2+k-l}{ i-2}({\rm e}_{k}+{\rm e}_{k-1})\\
		=&  \sum_{k=1}^{j-i}\sum_{l=j-i+1-k}^{k}\tbinom{i-2+k-l}{ i-2}({\rm e}_{k}+{\rm e}_{k-1})+\sum_{l=1}^{j-i+1}\tbinom{j-1-l}{ i-2}({\rm e}_{j-i+1}+{\rm e}_{j-i})\\
		=& \sum_{k=\lceil \frac{j-i+1}{2}\rceil}^{j-i}\sum_{l=0}^{2k-(j-i+1)}\tbinom{i-2+l}{ i-2}({\rm e}_{k}+{\rm e}_{k-1}) 
		+\sum_{l=0}^{j-i}\tbinom{i-2+l}{ i-2}({\rm e}_{j-i+1}+{\rm e}_{j-i})\\
		=&  \sum_{k=\lfloor \frac{j-i}{2}\rfloor+1}^{j-i}\tbinom{i-1+2k-(j-i+1)}{i-1}({\rm e}_{k}+{\rm e}_{k-1})
		+\tbinom{j-1}{i-1}({\rm e}_{j-i+1}+{\rm e}_{j-i})\\
		=& \tbinom{i+2\lfloor\frac{j-i}{2}\rfloor-(j-i)}{i-1}{\rm e}_{\lfloor \frac{j-i}{2}\rfloor} + \sum_{k=\lfloor \frac{j-i}{2}\rfloor+1}^{j-i-1}\left( \tbinom{i-1+2k-(j-i+1)}{i-1} +\tbinom{i-1+2k-(j-i+1)+2}{i-1} \right){\rm e}_{k}\\
		&+\left( \tbinom{j-1}{i-1}+\tbinom{j-2}{i-1} \right){\rm e}_{j-i}+\tbinom{j-1}{ i-1}{\rm e}_{j-i+1}
	\end{align*}
	Furthermore
		\begin{align*}
		&\;\;\;\sum_{\underset{\textnormal{$k$ even}}{k=1}}^{j-i+1}K_{i-2,j-k-1}({\rm e}_{k/2-1})
		=\sum_{k=1}^{\lfloor (j-i+1)/2\rfloor}K_{i-2,j-2k-1}({\rm e}_{k-1})\\
		&= \sum_{k=1}^{\lfloor (j-i+1)/2\rfloor}\tbinom{j-2k-1}{ i-2} \pi^{j-2k-1-i+2}({\rm e}_{k-1})
		= \sum_{k=1}^{\lfloor (j-i+1)/2\rfloor}\tbinom{j-2k-1}{ i-2} {\rm e}_{j-k-i}\\
		&=  \sum_{k=-(j-i)+1}^{\lfloor (j-i+1)/2\rfloor-(j-i)}\tbinom{2i-j-1-2k}{ i-2} {\rm e}_{-k}
		=  \sum_{k=(j-i)-\lfloor \frac{j-i+1}{2}\rfloor}^{j-i-1}\tbinom{i-1+2k-(j-i)}{ i-2} {\rm e}_{k}=  \sum_{k=\lfloor \frac{j-i}{2}\rfloor}^{j-i-1}\tbinom{i-1+2k-(j-i)}{ i-2} {\rm e}_{k}
		\end{align*}
		and
		\begin{align*}
			&\sum_{k=2}^{i}K_{i-k,j-1}(\gamma^*_{k,1}) =\sum_{k=2}^i\tbinom{j-1}{ i-k}\pi^{j-1-i+k}({\rm e}_{0}+{\rm e}_{1}) = \sum_{k=2}^{i}\tbinom{j-1}{ i-k}({\rm e}_{j-1-i+k}+{\rm e}_{j-i+k})\\
			\label{eq:secondTermExplicit}
			=& \sum_{k=2+(j-i-1)}^{i+(j-i-1)}\tbinom{j-1}{ j-1-k}({\rm e}_k+{\rm e}_{k+1}) = \tbinom{j-1}{ i-2}{\rm e}_{j-i+1}+\sum_{k=j-i+2}^{j}\tbinom{j}{ j-k}{\rm e}_k.
		\end{align*}
		If we plug the above expansions into equation~\eqref{eq:proofExplicit}, we obtain 
	\begin{align*}
		\gamma^*_{p_0,p_1}=&\; \mathds{1}_\mathbb{N}(\tfrac{p_1-p_0}{2}){\rm e}_{ \frac{p_1-p_0}{2}} + \sum_{k=\lfloor \frac{p_1-p_0}{2}\rfloor+1}^{p_1-p_0-1}\left( \tbinom{2p_0+2k-p_1-2}{p_0-1}-\tbinom{2p_0+2k-p_1-1}{p_0-2} +\tbinom{2p_0+2k-p_1}{p_0-1} \right){\rm e}_{k}\\
		&+\left(\tbinom{p_1-2}{p_0-1}+\tbinom{p_1-1}{p_0-1} \right){\rm e}_{p_1-p_0} +\sum_{k=p_1-p_0+1}^{p_1}\tbinom{p_1}{ p_1-k}{\rm e}_k,
	\end{align*}
	which is the claimed result.
	\end{proof}
 \begin{proof}[Proof of Lemma~\ref{lem:gammastar}]
		 The first condition of the bound condition requires $\tau_{p_0}^{p_1}\preceq\gamma^*_{p_0,p_1}$ for every $p_1\ge p_0\ge 1$.  
		 We prove the result by induction. The claim holds for $p_1\ge 1$ and $p_0=1$ by Proposition~\ref{prop:boundConditionInputOne}.  If we assume $p_1\ge p_0\ge 2$, $\tau_{p_0}^{p_1}\preceq\gamma^*_{p_0,p_1}$ and $\tau_{p_0+1}^{p_1}\preceq\gamma^*_{p_0+1,p_1}$ then Proposition~\ref{prop:recursion}, Lemma~\ref{lem:explicitTauPowerset} and equation~\eqref{eq:Pipreceq} imply that
	 \begin{equation*}
	 \tau_{p_0+1}^{p_1+1}\preceq \pi(\tau_{p_0+1}^{p_1})+\tau_{p_0}^{p_1}=\pi(\tau_{\min(p_1,p_0+1)}^{p_1})+\tau_{p_0}^{p_1}\preceq \pi(\gamma^*_{\min(p_0+1,p_1),p_1})+\gamma^*_{p_0,p_1}=\gamma^{*}_{p_0+1,p_1+1}
	 \end{equation*}
	 The second requirement of the bound condition is satisfied by a double induction: First do an induction for $p_1\ge p_0$ increasing and $p_0$ fixed and then an outer induction on $p_0$ increasing.
	 \end{proof}

	 In order to compare the unfolded collection $\gamma^*$ with the collection $\bar{\gamma}$ we show that the latter satisfies the same recursion property. 
 \begin{lemma}
	 \label{lem:recursionBinomialGamma}
	 For $n,p_1\in\mathbb{N}_{\ge 2}$ the identity $\bar{\gamma}_{n,p_1}=\pi\left( \bar{\gamma}_{\min(n,p_1-1),p_1-1} \right)+\bar{\gamma}_{n-1,p_1-1}$ holds.
	 \begin{proof}
		For all $i\in\mathbb{N}$ Pascal's identity implies 
	 \begin{align*}
		 \left( \bar{\gamma}_{n,p_1} \right)_i=&\; \sum_{j=0}^{n}\tbinom{p_1}{ j}\delta_{i,p_1-j}=\tbinom{p_1}{ p_1-i}\mathds{1}_{ \left\{ p_1-n,\dots,p_1 \right\}}(i)\\
		 =&\; \tbinom{p_1-1}{ p_1-i}\mathds{1}_{ \left\{ \max(p_1-n,1),\dots,p_1 \right\}}(i)+\tbinom{p_1-1}{ p_1-(i+1)}\mathds{1}_{ \left\{ p_1-n,\dots,p_1-1 \right\}}(i)\\
		 =&\; \tbinom{p_1-1}{ (p_1-1)-(i-1)}\mathds{1}_{ \left\{ (p_1-1)-\min(n,p_1-1),\dots,(p_1-1) \right\}}(i-1)\\
		 &\;+\tbinom{p_1-1}{ (p_1-1)-i}\mathds{1}_{ \left\{ (p_1-1)-(n-1),\dots,(p_1-1) \right\}}(i)\\
		 =&\; \left( \bar{\gamma}_{\min(n,p_1-1),p_1-1} \right)_{i-1}+\left(\bar{\gamma}_{n-1,p_1-1}  \right)_i =\left( \pi\left( \bar{\gamma}_{\min(n,p_1-1),p_1-1} \right)+\bar{\gamma}_{n-1,p_1-1} \right)_i\qedhere
	 \end{align*}
	 \end{proof}
 \end{lemma} 
 The fact that $\bar{\gamma}$ and $\gamma^*$ satisfy the same recursion property allows to conclude that $\gamma^*$ is at least as tight since this is true for input dimension $1$. This is formalized below.
	 \begin{proof}[Proof of Lemma~\ref{lem:newtighter}]
		 Since $\bar\gamma$ satisfies the bound condition, Definition~\ref{def:improvedCollection} implies for all $p_1\in\mathbb{N}_+$ the relation $\gamma^*_{1,p_1}=\tau_{1}^{p_1}\preceq\bar\gamma_{1,p_1}$. The result now follows by induction since both, $\gamma^*$ and $\bar{\gamma}$ satisfy the same recursive relation by the same definition and Lemma \ref{lem:recursionBinomialGamma} so the relation is inherited.
         %
	 \end{proof}
  \section{Results for the composition of subnetwork activation histogram bounds}
  \subsection{Basic auxiliary results}
  We begin with several obvious results.
 \begin{lemma}
	 \label{lem:preceqSum}
	 Let $v,w,v',w'\in V$. If $v\preceq v'$ and $w\preceq w'$ then $v+w\preceq v'+w'$.
 \end{lemma}
  \begin{lemma}
  \label{lem:preceqNorm}
  Assume for $v,w\in V$ that $v\preceq w$. Then $\|v\|_1\le \|w\|_1$.
\end{lemma}

\begin{lemma}
  \label{lem:preceqMultipleSummands}
  For $m\in \mathbb{N}_+$ and $a_1,\dots,a_m, b_1,\dots,b_m\in V$ it holds that 
  \begin{equation*}
	  \forall i\in \left\{ 1,\dots,m \right\}\;a_i\preceq b_i \implies \sum_{i=1}^{m} a_i\preceq\sum_{i=1}^{m}b_i.
  \end{equation*}
\end{lemma}
  \begin{lemma}
	  \label{lem:keepRelation}
	  Let $v_1,v_2\in V$ and $i^*\in\mathbb{N}$. It holds that $v_1\preceq v_2\implies \textnormal{cl}_{i^*}(v_1)\preceq\textnormal{cl}_{i^*}(v_2)$.
  \end{lemma}
\begin{lemma}
  \label{lem:monotonicityPhi}
  For $l\in\mathbb{N}_+$, $\mathbf{p}\in\mathbb{N}^l_+$, $\gamma^{\mathbf{p}}\in\Gamma^{\mathbf{p}}$,  and $v_1,v_2\in V$ the following monotonicity holds:
  \begin{equation*}
	  v_1\preceq v_2\implies \varphi^{(\gamma^{\mathbf{p}})}(v_1)\preceq \varphi^{(\gamma^{\mathbf{p}})}(v_2)
  \end{equation*}
  \begin{proof}
	This follows from equation~\eqref{eq:generalPhiFunc}, the second property of the bound condition of Definition~\ref{def:generalBoundCondition}, and Lemmas~\ref{lem:keepRelation} and \ref{lem:preceqMultipleSummands}.
  \end{proof}
\end{lemma}
\subsection{Derivation of the main result}
For this section, we assume that $m,l_1,\ldots,l_m, r_1,\ldots,r_m$, $\mathbf{p}_1,\ldots,\mathbf{p}_m$ and $\gamma^{\mathbf{p}_1},\ldots,\gamma^{\mathbf{p}_m} $ are defined as in Section~\ref{sec:mainResult}. We first show that the activation histogram join is replicated for input dimension larger than the first hidden layer dimension of the subnetwork.
  \begin{lemma}
  \label{lem:replication}
  For $p_0,l\in\mathbb{N}_{+}$, $\mathbf{p}=(p_1,\dots,p_l)\in\mathbb{N}^l_{+}$ and $p_0> p_1$, it holds that 
  $\tau^{\mathbf{p}}_{p_0}=\tau^{\mathbf{p}}_{p_1}$.
  \begin{proof}Similarly to Lemma~\ref{lem:tauBoundCondition} it holds that $\tau^{\mathbf{p}}_{\tilde p_0}\preceq \tau^{\mathbf{p}}_{\tilde p_0+1}$ for all $\tilde p_0\in\mathbb{N}_+$, in particular $\tau^{\mathbf{p}}_{p_1}\preceq \tau^{\mathbf{p}}_{p_0}$. It remains to show that $\tau_{p_0}^{\mathbf{p}}\preceq\tau_{p_1}^{\mathbf{p}}$. To this end, take $\mathbf{h}=(h_1,\dots,h_l)\in\textnormal{RL}(p_0,\mathbf{p})$. By definition $h_1\in\textnormal{RL}(p_0,p_1)$. Since $h_1$ has only $p_1<p_0$ neurons there exists a $p_1$-dimensional affine subspace $U\subset \mathbb{R}^{p_0}$ such that for every $x\in\mathbb{R}^{p_0}$ there exists $x_U\in U$ with $h_1(x)=h_1(x_U)$. With a bijective affine linear map $\Phi:\mathbb{R}^{p_1}\to U$ and $\tilde{\mathbf{h}}=(h_1\circ \Phi,h_2,\dots,h_l)\in\textnormal{RL}(p_1,\mathbf{p})$ it follows that $\mathcal{S}_{\mathbf{h}}=\mathcal{S}_{\tilde{\mathbf{h}}}$. Since $\mathbf{h}$ was arbitrary, $\tau_{p_0}^{\mathbf{p}}\preceq \tau_{p_1}^{\mathbf{p}}$ as required.
  \end{proof}
\end{lemma}
The above result makes clear why the first index in Definition \ref{def:generalBoundCondition} ranges only up to $p_1$: Since $\tau_{p_0}^{\mathbf{p}}$ is equalt to $\tau_{p_1}^{\mathbf{p}}$ for indices $p_0>p_1$ a collection of elementary bounds only needs to consider the input dimension up to $p_1$.
\begin{corollary}
  \label{cor:gammaMin}
  Assume that $l\in\mathbb{N}_+$. Then
\begin{equation*}
	\forall p_0\in\mathbb{N}_{+},\mathbf{p}=(p_1,\dots,p_l)\in\mathbb{N}_+^l, \gamma^{\mathbf{p}}\in\Gamma^{\mathbf{p}}:\quad \tau_{p_0}^{\mathbf{p}}\preceq \gamma^{\mathbf{p}}_{\min(p_0,p_1)}.
\end{equation*}
\begin{proof}
	We only need to consider the case where $p_0\ge p_1$. In this case, the result follows from Lemma~\ref{lem:replication}.
\end{proof}
\end{corollary}

To formalize the proofs below, we define the set of attained activation patterns in the first $l$ layers for $l\in\{1,\ldots,L\}$.
 \begin{definition}
	 For $l\in\left\{ 1,\dots,L \right\}$ let $ \mathcal{S}_{\mathbf{h}}^{(l)}=\left\{ (s_1,\dots,s_l)\middle\vert s\in\mathcal{S}_{\mathbf{h}} \right\}\subset \left\{ 0,1 \right\}^{n_1}\times\dots\times \left\{ 0,1 \right\}^{n_l}.$
 \end{definition}
 The following lemma is the basic building block for the proof of the main result Theorem~\ref{thm:generalMainResultPhi}. It bounds the activation histogram attained in subnetworks of the full network and makes use of the previously chosen collection of histograms satisfying the bound condition for this subnetwork. The non-tight estimate used here is responsible for the composition loss described in Section~\ref{sec:compositionLoss}. Below we write $\left\{ 0,1 \right\}^{\mathbf{p}_i}=\{0,1\}^{r_{i-1}+1}\times\cdots\times\{0,1\}^{r_i}$ for $i\in\{1,\ldots m\}$.
\begin{lemma}
  \label{lem:boundKappaByGammaNew}
  Assume $i\in \left\{ 1,\dots,m \right\}$ and fix $(s_1^*,\dots,s_{r_{i-1}}^*)\in\mathcal{S}_{\mathbf{h}}^{(r_{i-1})}$. Then
  \begin{align*}
	  &\;\sum_{\overset{(s_{r_{i-1}+1},\dots,s_{r_i})}{\in\left\{ 0,1 \right\}^{\mathbf{p}_i}}}\mathds{1}_{\mathcal{S}_{\mathbf{h}}^{(r_i)}}\left( (s^*_1,\dots,s^*_{r_{i-1}},s_{r_{i-1}+1},\dots,s_{r_i} )\right){\rm e}_{\min(|s_{r_{i-1}+1}|,\dots,|s_{r_i}|)} \\
  \preceq&\;  \gamma_{\min\left( n_0,|s_1^*|,\dots,|s_{r_{i-1}}^*|,n_{r_{i-1}+1} \right)}^{\mathbf{p}_i}.
  \end{align*}
  \begin{proof}
	  With $R=\big\{ x\in\mathbb{R}^{n_0}\vert\;S_{\mathbf{h}}(x)_j=s_j^*\textnormal{ for }j\in\{1,\dots,r_{i-1}\} \big\}$, $\mathbf{g}=(h_{r_{i-1}+1},\dots,h_{r_i})$ and $f: R\to \mathbb{R}^{n_{r_{i-1}}}, x\mapsto h_{r_i-1}\circ\dots\circ h_1(x)$ it holds that
  \begin{align*}
	  &\;\sum_{\overset{(s_{r_{i-1}+1},\dots,s_{r_i})}{\in\left\{ 0,1 \right\}^{\mathbf{n}_i}}}\mathds{1}_{\mathcal{S}_{\mathbf{h}}^{(r_i)}}\left( (s^*_1,\dots,s^*_{r_{i-1}},s_{r_{i-1}+1},\dots,s_{r_i} )\right){\rm e}_{\min(|s_{r_{i-1}+1}|,\dots,|s_{r_i}|)} \\
	  =&\; \sum_{
		  \overset{(s_{r_{i-1}+1},\dots,s_{r_i})\in}{\{(S_{\mathbf{h}}(x)_{r_{i-1}+1},\ldots,S_{\mathbf{h}}(x)_{r_i} )|x\in R\}}
  }{\rm e}_{\min(|s_{r_{i-1}+1}|,\dots,|s_{r_{i}}|)}\\
	  =&\; \sum_{
		  \overset{(s_{r_{i-1}+1},\dots,s_{r_i})\in}{\{(S_{\mathbf{g}}(f(x))_{1},\ldots,S_{\mathbf{g}}(f(x))_{l_i} )|x\in R\}}
  }{\rm e}_{\min(|s_{r_{i-1}+1}|,\dots,|s_{r_{i}}|)} \\
	  =&\; \sum_{
		  \overset{(s_{r_{i-1}+1},\dots,s_{r_i})\in}{\{(S_{\mathbf{g}}(f(x))_{1},\ldots,S_{\mathbf{g}}(f(x))_{l_i} )|x\in R\}}
	  }{\rm e}_{\min(|s_{r_{i-1}+1}|,\dots,|s_{r_{i}}|)}=:(*)
\end{align*} 
If we allow all inputs $x\in\mathbb{R}^{n_0}$ instead of $x\in R$ the resulting histogram can only get larger with respect to $\preceq$. Note that $f$ is affine linear because its domain is a region with constant activation pattern for the first $r_{i-1}$ layers. Let $\tilde f$ be its affine linear extension to $\mathbb{R}^{n_0}$. Since $\tilde f$ is affine linear, there exists a bijective affine linear map $\Phi:\mathbb{R}^{\textnormal{rank}(\tilde f)}\to \tilde f(\mathbb{R}^{n_0})$. With
$\tilde{\mathbf{g}}=(h_{r_{i-1}+1}\circ\Phi,h_{r_{i-1}+2}\dots,h_{r_i})$ it holds that
  \begin{align*}
	  (*)\preceq&\; \sum_{
		  \overset{(s_{r_{i-1}+1},\dots,s_{r_i})\in}{\{(S_{\mathbf{g}}(\tilde f(x))_{1},\ldots,S_{\mathbf{g}}(\tilde f(x))_{l_1} )|x\in \mathbb{R}^{n_0}\}}
	  }{\rm e}_{\min(|s_{r_{i-1}+1}|,\dots,|s_{r_{i}}|)}\\
	  =&\; \sum_{
		  \overset{(s_{r_{i-1}+1},\dots,s_{r_i})\in}{\{(S_{\tilde{\mathbf{g}}}(x)_{1},\ldots,S_{\tilde{\mathbf{g}}}(x)_{l_1} )|x\in \mathbb{R}^{\textnormal{rank}(\tilde f)}\}}
	  }{\rm e}_{\min(|s_{r_{i-1}+1}|,\dots,|s_{r_{i}}|)}=\sum_{s\in\mathcal{S}_{\tilde{\mathbf{g}}}}{\rm e}_{\min(|s_1|,\dots,|s_{l_i}|)}\preceq\tau_{ \textnormal{rank}(\tilde f)}^{\mathbf{p}_i},
\end{align*} 
where the last inequality holds by the Definition~\ref{def:multitau} of $\tau$ and $\tilde{\mathbf{g}}\in\textnormal{RL}(\textnormal{rank}(\tilde f),\mathbf{p}_i)$. By the first property of the bound condition from Definition~\ref{def:generalBoundCondition} and by Corollary~\ref{cor:gammaMin} it follows $\tau_{\textnormal{rank}(\tilde f)}^{\mathbf{p}_i}\preceq\gamma_{\min\left( \textnormal{rank}(\tilde f),n_{r_{i-1}+1} \right)}^{\mathbf{p}_i}$.
Now note that $\textnormal{rank}(\tilde f)$ is bounded by $ \min(n_0,|s^*_1|,\dots,|s^*_{r_{i-1}}|)$ which concludes the proof by the second property of the bound condition.
  \end{proof}
\end{lemma}
\begin{definition}
  For $l\in\left\{ 1,\dots,L \right\}$, define the \emph{dimension histogram} by
  \begin{equation*}
  {\mathcal{H}}^{(l)}_{\mathbf{h}}=\sum_{(s_1,\dots,s_l)\in S^{(l)}_{\mathbf{h}}}{\rm e}_{\min(n_0,|s_1|,\dots,|s_l|) }
  \end{equation*}
\end{definition}
Below we present the final two ingredients to prove Theorem~\ref{thm:generalMainResultPhi}. The following result plays a similar role as the anchor in an induction proof.
\begin{lemma} It holds that ${\mathcal{H}}^{(r_1)}_{\mathbf{h}}\preceq\varphi^{(\gamma^{\mathbf{p}_1})}({\rm e}_{n_0})$.
 \label{lem:main2New} 
   \begin{proof}
	   Note that $\mathbf{g}:=(h_1,\dots,h_{r_1})\in\textnormal{RL}(n_0,\mathbf{p}_1)$ such that Lemma~\ref{lem:keepRelation}
	 and Corollary~\ref{cor:gammaMin}
	 imply 
	 \begin{equation*}
		 {\mathcal{H}}^{(r_1)}_{\mathbf{h}}=\textnormal{cl}_{\min(n_0,n_1)}(\sum_{s\in\mathcal{S}^{(r_1)}_{\mathbf{h}}}{\rm e}_{\min(|s_1|,\dots,|s_{r_1}|)|})\preceq\textnormal{cl}_{\min(n_0,n_1)}\left( \gamma_{\min(n_0,n_1)}^{\mathbf{p}_1} \right),
	 \end{equation*}
	 where the last expression is equal to $\varphi^{(\gamma^{\mathbf{p}_1})}\left( {\rm e}_{n_0} \right)$ by definition in equation~\eqref{eq:generalPhiFunc}
   \end{proof}
 \end{lemma}
The following second incredient plays a similar role as the inductive step in an induction proof.
 \begin{proposition}
   \label{prop:main2}
	For $i\in\left\{ 2,\dots,m \right\}$ it holds that ${\mathcal{H}}^{(r_i)}_{\mathbf{h}}\preceq \varphi^{(\gamma^{\mathbf{p}_i})}({\mathcal{H}}^{(r_{i-1})}_{\mathbf{h}})$.
   \begin{proof} The statement is implied by the following calculation, where the step indicated by $(*)$ follows from the Lemmas~\ref{lem:boundKappaByGammaNew}, \ref{lem:keepRelation} and \ref{lem:preceqMultipleSummands}.
\begin{align*}
{\mathcal{H}}^{(r_i)}_{\mathbf{h}}=& \sum_{\overset{(s_1,\dots,s_{r_{i-1}})}{\in\mathcal{S}_{\mathbf{h}}^{(r_{i-1})}}}^{}\sum_{\overset{(s_{r_{i-1}+1},\dots,s_{r_i})}{\in\left\{ 0,1 \right\}^{\mathbf{p}_i}}}^{}\mathds{1}_{\mathcal{S}_{\mathbf{h}}^{(r_i)}}\left( (s_1,\dots,s_{r_i} )\right) {\rm e}_{  \left( \min(n_0,|s_1|,\dots,|s_{r_i}|) \right)}\\
=&  \sum_{\overset{(s_1,\dots,s_{r_{i-1}})}{\in\mathcal{S}_{\mathbf{h}}^{(r_{i-1})}}}\textnormal{cl}_{\min(n_0,|s_1|,\dots,|s_{r_{i-1}}|,n_{r_{i-1}+1})}\bigg(\underbrace{ \sum_{\overset{(s_{r_{i-1}+1},\dots,s_{r_i})}{\in\left\{ 0,1 \right\}^{\mathbf{p}_i}}}\mathds{1}_{\mathcal{S}_{\mathbf{h}}^{(r_i)}}\left( (s_1,\dots,s_{r_i} )\right){\rm e}_{\min(|s_{r_{i-1}+1}|,\dots,|s_{r_i}|)}}_{\preceq \gamma_{\min(n_0,|s_1|,\dots,|s_{r{i-1}}|,n_{r_{i-1}+1})}^{\mathbf{p}_i}}  \bigg)\\
	\overset{(*)}{\preceq}&  \sum_{(s_1,\dots,s_{r_{i-1}})\in\mathcal{S}_{\mathbf{h}}^{(r_{i-1})}}^{}\textnormal{cl}_{\min(n_0,|s_1|,\dots,|s_{r_{i-1}}|,n_{r_{i-1}+1})}(\gamma_{\min(n_0,|s_1|,\dots,|s_{r_{i-1}}|,n_{r_{i-1}+1})}^{\mathbf{p}_i}  )\\
=&  \sum_{j=0}^{\infty}\left( {\mathcal{H}}^{(r_{i-1})}_{\mathbf{h}} \right)_j\textnormal{cl}_{\min(j,n_{r_{i-1}+1})}\left(\gamma_{\min(j,n_{r_{i-1}+1})}^{\mathbf{p}_i}\right)=  \sum_{j=0}^{\infty}\left( {\mathcal{H}}^{(r_{i-1})}_{\mathbf{h}} \right)_j \varphi^{(\gamma^{\mathbf{p}_i})}({\rm e}_j)\\
	=& \;  \varphi^{(\gamma^{\mathbf{p}_i})}\left( \sum_{j=0}^{\infty}\left( {\mathcal{H}}^{(r_{i-1})}_{\mathbf{h}} \right)_j{\rm e}_j \right)=\varphi^{(\gamma^{\mathbf{p}_i})}\left( {\mathcal{H}}^{(r_{i-1})}_{\mathbf{h}} \right).\qedhere
\end{align*}
   \end{proof}
 \end{proposition}
  \begin{proof}[Proof of Theorem~\ref{thm:generalMainResultPhi}]
	  First note that $L=r_m$ and that $\vert\mathcal{S}_{\mathbf{h}}^{(r_m)}|=\|{\mathcal{H}}^{(r_m)}_{\mathbf{h}}\|_1$. Now the above Proposition~\ref{prop:main2} and Lemmas~\ref{lem:preceqNorm}, ~\ref{lem:monotonicityPhi} and \ref{lem:main2New} imply that
 \begin{align*}
	 |\mathcal{S}_{\mathbf{h}}|&=  |\mathcal{S}_{\mathbf{h}}^{(r_m)}|=\|{\mathcal{H}}^{(r_m)}_{\mathbf{h}}\|_1\le\|\varphi^{(\gamma^{\mathbf{p}_m})}({\mathcal{H}}^{(r_{m-1})}_{\mathbf{h}})\|_1\le\dots\\
	 &\le \|\varphi^{(\gamma^{\mathbf{p}_m})}\circ\dots\circ\varphi^{(\gamma^{\mathbf{p}_2})}({\mathcal{H}}^{(r_1)}_{\mathbf{h}})\|_1\le\|\varphi^{(\gamma^{\mathbf{p}_m})}\circ\dots\circ\varphi^{(\gamma^{\mathbf{p}_1})}({\rm e}_{n_0})\|_1.
 \end{align*}The matrix formulation easily follows from the insight that all information of the map $\varphi^{(\gamma^{\mathbf{p}_k})}$ can be encoded in the $(p_{r_{k-1}+1}+1)\times(p_{r_{k-1}+1}+1)$ matrix $B^{(\gamma^{\mathbf{p}_k})}$ (in equation~\eqref{eq:generalPhiFunc} the minima  bound the number or columns needed, the clipping function ``$\textnormal{cl}$'' bounds the number of rows needed). The ``+1'' stems from the fact that indexing starts with $0$ in $V$ but with $1$ for matrices.
 \end{proof}
\newpage

\bibliography{obereExtension} 

\begin{thebibliography}{10}

\bibitem{hinz}
P.~{Hinz} and S.~{van de Geer}, ``A framework for the construction of upper
  bounds on the number of affine linear regions of relu feed-forward neural
  networks,'' {\em IEEE Transactions on Information Theory}, vol.~65,
  pp.~7304--7324, Nov 2019.

\bibitem{zhang2020empirical}
X.~Zhang and D.~Wu, ``Empirical studies on the properties of linear regions in
  deep neural networks,'' {\em arXiv preprint arXiv:2001.01072}, 2020.

\bibitem{pmlr-v119-xiong20a}
H.~Xiong, L.~Huang, M.~Yu, L.~Liu, F.~Zhu, and L.~Shao, ``On the number of
  linear regions of convolutional neural networks,'' in {\em Proceedings of the
  37th International Conference on Machine Learning} (H.~D. III and A.~Singh,
  eds.), vol.~119 of {\em Proceedings of Machine Learning Research},
  pp.~10514--10523, PMLR, 13--18 Jul 2020.

\bibitem{hanin2019complexity}
B.~Hanin and D.~Rolnick, ``Complexity of linear regions in deep networks,'' in
  {\em International Conference on Machine Learning}, pp.~2596--2604, PMLR,
  2019.

\bibitem{hanin2019deep}
B.~Hanin and D.~Rolnick, ``Deep relu networks have surprisingly few activation
  patterns,'' {\em arXiv preprint arXiv:1906.00904}, 2019.

\bibitem{HowToStartTraining}
B.~Hanin and D.~Rolnick, ``How to start training: The effect of initialization
  and architecture,'' in {\em Advances in Neural Information Processing
  Systems} (S.~Bengio, H.~Wallach, H.~Larochelle, K.~Grauman, N.~Cesa-Bianchi,
  and R.~Garnett, eds.), vol.~31, Curran Associates, Inc., 2018.

\bibitem{ImpactActivation}
S.~Hayou, A.~Doucet, and J.~Rousseau, ``On the impact of the activation
  function on deep neural networks training,'' in {\em Proceedings of the 36th
  International Conference on Machine Learning} (K.~Chaudhuri and
  R.~Salakhutdinov, eds.), vol.~97 of {\em Proceedings of Machine Learning
  Research}, pp.~2672--2680, PMLR, 09--15 Jun 2019.

\bibitem{DNNDecisionTrees}
K.~D. {Humbird}, J.~L. {Peterson}, and R.~G. {Mcclarren}, ``Deep neural network
  initialization with decision trees,'' {\em IEEE Transactions on Neural
  Networks and Learning Systems}, vol.~30, no.~5, pp.~1286--1295, 2019.

\bibitem{kumar2017weight}
S.~K. Kumar, ``On weight initialization in deep neural networks,'' {\em arXiv
  preprint arXiv:1704.08863}, 2017.

\bibitem{Zou2020}
D.~Zou, Y.~Cao, D.~Zhou, and Q.~Gu, ``Gradient descent optimizes
  over-parameterized deep relu networks,'' {\em Machine Learning}, vol.~109,
  pp.~467--492, Mar 2020.

\bibitem{CiCP-28-1671}
L.~Lu, Y.~Shin, Y.~Su, and G.~Em~Karniadakis, ``Dying relu and initialization:
  Theory and numerical examples,'' {\em Communications in Computational
  Physics}, vol.~28, no.~5, pp.~1671--1706, 2020.

\bibitem{Montufar:2014:NLR:2969033.2969153}
G.~Mont\'{u}far, R.~Pascanu, K.~Cho, and Y.~Bengio, ``On the number of linear
  regions of deep neural networks,'' in {\em Proceedings of the 27th
  International Conference on Neural Information Processing Systems - Volume
  2}, NIPS'14, (Cambridge, MA, USA), pp.~2924--2932, MIT Press, 2014.

\bibitem{DBLP:BoundingCounting}
T.~Serra, C.~Tjandraatmadja, and S.~Ramalingam, ``Bounding and counting linear
  regions of deep neural networks,'' {\em CoRR}, vol.~abs/1711.02114, 2017.

\bibitem{Schlafli1950}
L.~Schl{\"a}fli, {\em Theorie der vielfachen Kontinuit{\"a}t}, pp.~209--212.
\newblock Basel: Springer Basel, 1950.

\bibitem{Montufar17}
G.~Mont\'{u}far, ``Notes on the number of linear regions of deep neural
  networks,'' 03 2017.

\bibitem{xie2020general}
Y.~Xie, G.~Chen, and Q.~Li, ``A general computational framework to measure the
  expressiveness of complex networks using a tighter upper bound of linear
  regions,'' {\em arXiv preprint arXiv:2012.04428}, 2020.

\bibitem{buck1943partition}
R.~C. Buck, ``Partition of space,'' {\em American Mathematical Monthly},
  vol.~50, no.~9, pp.~541--544, 1943.

\end{thebibliography}
\bibliographystyle{ieeetr}
  \end{document}